\documentclass[letterpaper]{article} 
\usepackage{aaai2026}  
\usepackage{times}  
\usepackage{helvet}  
\usepackage{courier}  
\usepackage[hyphens]{url}  
\usepackage{graphicx} 
\usepackage{amsmath}
\usepackage{amsfonts}
\usepackage{natbib}  
\usepackage{caption} 
\usepackage{algorithm}
\usepackage[noEnd]{algpseudocodex}
\usepackage{amsthm}

\usepackage{comment}

\usepackage[short,c2,nocomma]{optidef}

\usepackage{newfloat}
\usepackage{listings}
\DeclareCaptionStyle{ruled}{labelfont=normalfont,labelsep=colon,strut=off} 
\floatstyle{ruled}
\newfloat{listing}{tb}{lst}{}
\floatname{listing}{Listing}
\def\Tar{{\textnormal{Tar}}}
\def\ntar{{n_\textnormal{tar}}}
\def\nagt{{n_\textnormal{agt}}}
\def\PartialTour{{\textnormal{PartialTour}}}

\def\LFDT{{\textnormal{LFDT}}}
\def\EFAT{{\textnormal{EFAT}}}
\def\Length{{\textnormal{Len}}}
\def\taua{{\tau_\textnormal{a}}}
\def\cred{{c_\textnormal{red}}}
\def\cstarred{{c^*_\textnormal{red}}}

\def\nptour{{n}}

\DeclareMathOperator*{\argmin}{arg\,min}

\newcommand{\pvec}[1]{\vec{#1}\mkern2mu\vphantom{#1}}

\newtheorem{theorem}{Theorem}{\bfseries}{\itshape}
{\bfseries}{\itshape}
\newtheorem{lemma}{Lemma}{\bfseries}{\itshape}
\newtheorem{remark}{Remark}{\itshape}{\rmfamily}

\makeatletter
\newcommand{\subalign}[1]{%
  \vcenter{%
    \Let@ \restore@math@cr \default@tag
    \baselineskip\fontdimen10 \scriptfont\tw@
    \advance\baselineskip\fontdimen12 \scriptfont\tw@
    \lineskip\thr@@\fontdimen8 \scriptfont\thr@@
    \lineskiplimit\lineskip
    \ialign{\hfil$\m@th\scriptstyle##$&$\m@th\scriptstyle{}##$\hfil\crcr
      #1\crcr
    }%
  }%
}
\makeatother

\title{Optimal Solutions for the Moving Target Vehicle Routing Problem via Branch-and-Price with Relaxed Continuity}
\author{
    Anoop Bhat\textsuperscript{\rm 1},
    Geordan Gutow\textsuperscript{\rm 2},
    Zhongqiang Ren\textsuperscript{\rm 3}
    Sivakumar Rathinam\textsuperscript{\rm 4}
    Howie Choset\textsuperscript{\rm 1}
}
\affiliations {
    \textsuperscript{\rm 1}Robotics Institute at Carnegie Mellon University, Pittsburgh, PA 15213\\
    \textsuperscript{\rm 2}Mechanical and Aerospace Engineering at Michigan Technological University, Houghton, MI 49931\\
    \textsuperscript{\rm 3}UM-SJTU Joint Institute and Department of Automation at Shanghai Jiao Tong University, Shanghai, China\\
    \textsuperscript{\rm 4}Department of Mechanical Engineering and Department of Computer Science and Engineering at Texas A\&M University, College Station, TX 77843\\
    agbhat@andrew.cmu.edu, gmgutow@mtu.edu, zhongqiang.ren@sjtu.edu.cn, srathinam@tamu.edu, choset@andrew.cmu.edu.
}

\begin{document}

\maketitle

\begin{abstract}
 

The Moving Target Vehicle Routing Problem (MT-VRP) seeks trajectories for several agents that intercept a set of moving targets, subject to speed, time window, and capacity constraints. We introduce an exact algorithm, Branch-and-Price with Relaxed Continuity (BPRC), for the MT-VRP. The main challenge in a branch-and-price approach for the MT-VRP is the pricing subproblem, which is complicated by moving targets and time-dependent travel costs between targets. Our key contribution is a new labeling algorithm that solves this subproblem by means of a novel dominance criterion tailored for problems with moving targets. Numerical results on instances with up to 25 targets show that our algorithm finds optimal solutions more than an order of magnitude faster than a baseline based on previous work, showing particular strength in scenarios with limited agent capacities.

\end{abstract}


\section{Introduction}\label{sec:intro}
The Vehicle Routing Problem (VRP) is one of the most well-studied problems in logistics, with numerous applications \cite{toth2014vehicle,archetti2025beyond}. The standard VRP considers a set of targets, each with a demand for goods, and a fleet of agents (often called vehicles), each with a capacity limit on the amount of goods it can deliver. Given the travel costs between every pair of targets, as well as between each target and the depot (the starting location of the agents), the VRP seeks a sequence of targets for each agent such that each target is visited by exactly one agent, the total demand serviced by each agent does not exceed its capacity, and the sum of the travel costs of the agents is minimized.

In this article, we consider a generalization of the VRP in which the targets are moving, and each target has one or more time windows in which it can be visited (Fig. \ref{fig:intro_fig}). We call this generalization the Moving Target VRP (MT-VRP). We are motivated to solve this generalization due to its applications, including defense \cite{helvig2003moving,smith2021assessment,stieber2022DealingWithTime}, shipping supplies to ships at sea \cite{brown2017scheduling}, aerial fueling \cite{barnes2004solving}, and recharging underwater vehicles monitoring the seafloor \cite{Li2019RendezvousPlanning}.

\begin{figure}
    \centering
    \includegraphics[width=0.47\textwidth]{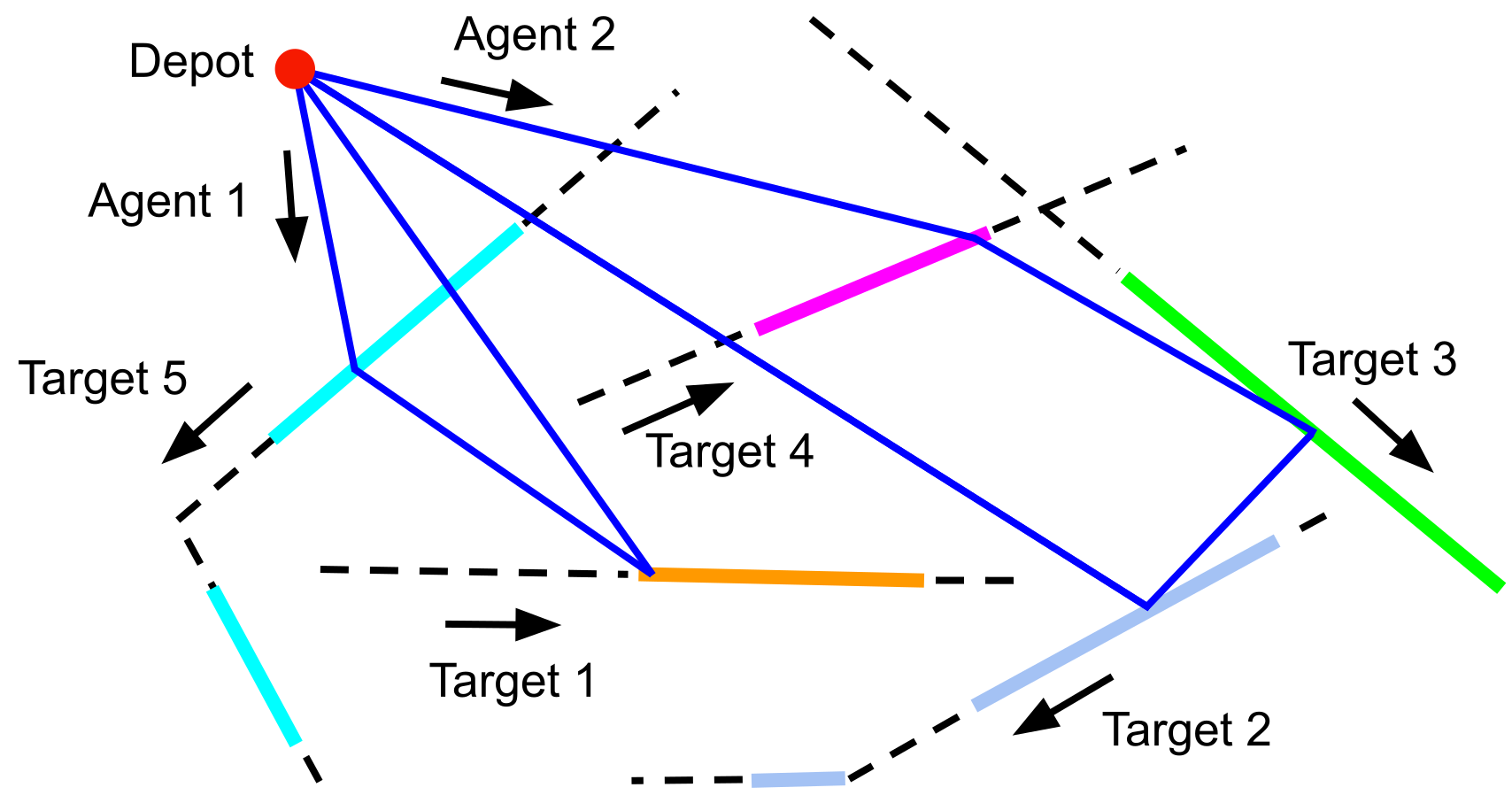}
    \caption{Targets move along piecewise-linear trajectories with time windows shown in bold lines. Two agents begin at the depot and collectively intercept all targets. Trajectories of agents are shown in blue.}
    \label{fig:intro_fig}
\end{figure}


The MT-VRP generalizes the Traveling Salesman Problem (TSP) and is therefore NP-hard \cite{hammar1999,helvig2003moving}. Although previous work has investigated simplified versions of the MT-VRP without capacity constraints \cite{philip2025Mixed, stieber2022DealingWithTime}, the capacitated version has not been considered and is the focus of this paper. Since state-of-the-art exact VRP solvers are based on the branch-and-price method \cite{costa2019exact}, we present a branch-and-price approach for the MT-VRP.

A key aspect of the branch-and-price method involves alternately solving a \emph{master problem} and a \emph{pricing problem.} The master problem seeks to assign each agent a single sequence of target-time window pairings from a subset of all possible sequences. We refer to a target-window pairing as a \emph{target-window}. The pricing problem attempts to add sequences to the subset of target-window sequences considered in the master problem. In particular, the pricing problem seeks one or more sequences with negative \emph{reduced cost}, which is the cost of the sequence (i.e., the travel distance required to visit its targets), plus additional bias terms, discussed in Section \ref{sec:pricing_problem}. Solving the pricing problem efficiently requires a method for identifying when a sequence $\Gamma$ \emph{dominates} another sequence $\Gamma'$ ending at the same target-window, meaning no extension of $\Gamma'$ can have a lower (i.e. more negative) reduced cost than the same extension of $\Gamma$. 

Typical VRP dominance checks assume that when we append target-windows to $\Gamma$, the cost incurred by visiting these target-windows only depends on the final target-window in $\Gamma$, and that the same goes for $\Gamma'$. Under this assumption, appending the same sequence of target-windows to $\Gamma$ and $\Gamma'$ would result in the same increase in cost. However, as shown in Fig.~\ref{fig:conventional_dominance_fails_intro}, this assumption breaks when we have moving targets. In this case, appending the same sequence of target-windows to $\Gamma$ and $\Gamma'$ actually increases the cost of $\Gamma$ more than for $\Gamma'$. This is because the cost of the sequence obtained by appending target-windows to $\Gamma$ is computed using a continuous trajectory optimization problem that depends on the entire sequence of target-windows in $\Gamma$, not just its final target-window.

To address this challenge, we introduce a new labeling algorithm with a dominance check for the MT-VRP that compares an upper bound on the cost of $\Gamma$ and a lower-bound on the cost of $\Gamma'$, which defers the expense of solving a continuous optimization problem until necessary. The upper bound is efficiently calculated by constructing a feasible trajectory using sampled points from the target-windows, whereas the lower bound is computed via a trajectory planning problem with relaxed continuity constraints, inspired by \cite{philip2025C}. We refer to our approach as Branch-and-Price with Relaxed Continuity (BPRC). We present numerical results that demonstrate in problems where agents have small capacities, BPRC achieves more than an order of magnitude speedup over a baseline method, which incorporates capacity constraints from \cite{desrochers1987vehicle} into a state-of-the-art multi-agent MT-TSP solver \cite{philip2025Mixed}. We also present ablation studies to further validate our approach.

\begin{figure}
    \centering
    \includegraphics[width=0.4\textwidth]{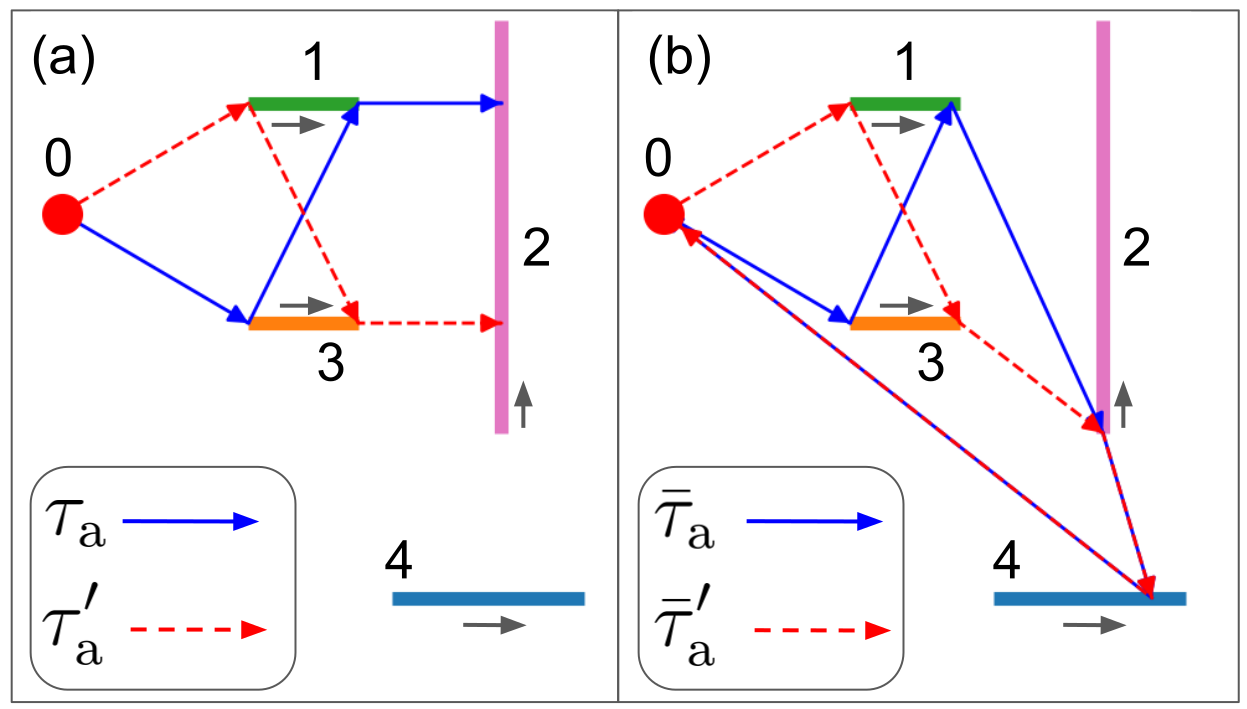}
    \caption{Example where existing VRP dominance checks fail for the MT-VRP. Since each target has one time window in this example, we represent a target-window using the index of the target. We treat the depot as a fictitious target 0. (a) $\taua$ is an optimal (i.e. minimum-distance) trajectory visiting the sequence of targets $\Gamma = (0, 3, 1, 2)$, and $\taua'$ is an optimal trajectory visiting the sequence $\Gamma' = (0, 1, 3, 2)$.\footnotemark{} Conventional VRP dominance checks would conclude that $\Gamma$ dominates $\Gamma'$, since $\taua$ and $\taua'$ have the same cost.\footnotemark{} (b) $\bar{\tau}_\text{a}$ optimally visits $\bar{\Gamma} = (0, 3, 1, 2, 4, 0)$ obtained by appending 4 and 0 to $\Gamma$, and $\bar{\tau}_\text{a}'$ optimally visits $\bar{\Gamma}' = (0, 1, 3, 2, 4, 0)$, obtained by appending 4 and 0 to $\Gamma'$. $\bar{\tau}_\text{a}$ travels more distance than $\bar{\tau}_\text{a}'$, so $\bar{\Gamma}$ is worse than $\bar{\Gamma}'$. Thus $\Gamma$ does not dominate $\Gamma'$.}
    \label{fig:conventional_dominance_fails_intro}
\end{figure}
\addtocounter{footnote}{-1}
\footnotetext{Agents in the MT-VRP can change their speed, which is why $\taua$ and $\taua'$ can travel the same distance, even though $\taua$ meets target 2 later than $\taua'$: $\taua$ can simply arrive at its interception position, then stop and wait for target 2.}
\addtocounter{footnote}{1}
\footnotetext{While dominance checks technically compare \emph{reduced cost}, which adds bias terms to the cost, the bias terms for $\Gamma$ and $\Gamma'$ are equal in Fig.~\ref{fig:conventional_dominance_fails_intro} (a), and the same goes for $\bar{\Gamma}$ and $\bar{\Gamma}'$ in Fig. \ref{fig:conventional_dominance_fails_intro} (b).}

\section{Related Work}
\subsubsection{VRP:}\label{sec:cvrp_related_work} Approaches for solving VRP variants have utilized two main types of integer programming formulations: \emph{compact} and \emph{extended}. In compact formulations, the number of variables and constraints is polynomial in the number of targets~\cite{desrochers1987vehicle,bard2002branch}. In contrast, extended formulations have a number of variables and constraints that is exponential~\cite{ozbaygin2017branch,desrochers1992new,costa2019exact}. Although compact formulations are smaller optimization problems, state-of-the-art VRP solvers, based on branch-and-price, leverage extended formulations. This is because the extended formulations tend to have tighter relaxations, i.e., the optimal cost of an extended formulation's convex relaxation tends to be closer to the optimal cost of the non-relaxed problem. Our work employs an extended formulation of the MT-VRP, and our numerical results in Section \ref{sec:numerical_results} show that the extended formulation tends to be tighter than a compact one.

\subsubsection{Multi-Agent MT-TSP:} Prior work~\cite{philip2025Mixed} finds optimal solutions for the Multi-Agent MT-TSP without capacity constraints using a mixed-integer conic program (MICP), assuming that targets move along piecewise-linear trajectories. We make the same piecewise-linear assumption in our work. Our baseline is based  on \cite{philip2025Mixed}. Note that, using the terminology from the previous section, the MICP from \cite{philip2025Mixed} is compact.

\subsubsection{VRP with Floating Targets:} The VRP with Floating Targets (VRPFT) \cite{gambella2018vehicle} is a problem similar to the MT-VRP, but where the movements of the targets must also be planned, and the targets do not have time windows. \cite{gambella2018vehicle} applies branch-and-price to the VRPFT, though the pricing problem is solved using an off-the-shelf mixed-integer program solver. In other branch-and-price methods, it is common to instead use a labeling algorithm for the pricing problem \cite{costa2019exact}. In contrast to \cite{gambella2018vehicle}, our work uses a labeling algorithm, introducing a new dominance criterion for moving targets.

\section{Problem Setup}\label{sec:problem_setup}
We consider $\nagt$ agents and $\ntar$ targets moving in  $\mathbb{R}^2$. All agents start at the same depot, with position $q_0 \in \mathbb{R}^2$. They also have an identical speed limit $v_\text{max} \in \mathbb{R}_{\geq 0}$, as well as an identical capacity $d_\text{max} \in \mathbb{R}_{\geq 0}$. 

Each target $i$ has a demand $d_i \in \mathbb{R}_{\geq 0}$ that must be met by exactly one agent. It can be visited by any agent during one of its time windows in $\{[\underline{t}_{i,1},\overline{t}_{i,1}], \dots, [\underline{t}_{i,n_\text{win}(i)}, \overline{t}_{i,n_\text{win}(i)}]\}$, where $[\underline{t}_{i,j}, \overline{t}_{i,j}] \subseteq \mathbb{R}_{\geq 0}$ is the $j$th time window of target $i$, and $n_\text{win}(i)$ is the number of time windows of target $i$. Each target $i$ moves along a trajectory $\tau_i: \bigcup\limits_{j = 1}^{n_\text{win}(i)}[\underline{t}_{i,j}, \overline{t}_{i,j}] \rightarrow \mathbb{R}^2$ and has a constant velocity within a time window, but possibly different velocities in different time windows. We assume no target at any time moves faster than $v_\text{max}$.

An agent's trajectory is \emph{speed-admissible} if the agent's speed at all times along its trajectory satisfies the speed limit. The cost of an agent trajectory $\taua$, denoted as $c(\taua)$, is its distance traveled, if $\taua$ is speed-admissible, and $\infty$ otherwise. An agent's trajectory $\taua$ \emph{intercepts} target $i$ if there exists a time $t$ in one of target $i$'s time windows such that (1) $\taua(t) = \tau_i(t)$, and (2) $\taua$ \emph{claims} target $i$ at time $t$.
The notion of ``claiming" is needed in scenarios where we plan a trajectory $\taua$ with the intent of intercepting target $i$, then $j$, but $\taua$ unintentionally satisfies $\taua(t) = \tau_{k}(t)$ for some target $k$ between the interceptions of $i$ and $j$.

The MT-VRP seeks a speed-admissible trajectory for each agent such that (i) each agent's trajectory begins and ends at the depot, (ii) every target is intercepted by exactly one agent, (iii) the sum of the demands of targets intercepted by each agent does not exceed $d_\text{max}$, and (iv) the sum of the agents' trajectory costs is minimized.

\section{Integer Linear Program (ILP) for MT-VRP}\label{sec:target_window_graph}
We formulate the MT-VRP on a graph $\mathcal{G}_\text{tw} = (\mathcal{V}_\text{tw}, \mathcal{E}_\text{tw})$ referred to as the \emph{target-window-graph}. Each node in $\mathcal{V}_\text{tw}$ is a target paired with one of its time windows. We call such a pairing a \emph{target-window}. The \emph{target-windows} for a target $i \in \mathcal{V}_\text{tw}$ are represented as $\gamma_{i,j}=(i, [\underline{t}_{i,j}, \overline{t}_{i,j}])$ for $j \in \{1, \dots, n_\text{win}(i)\}$. To simplify notation, we refer to the depot as a fictitious target $0$ with demand $d_0 = 0$ and trajectory $\tau_0$ satisfying $\tau_0(t) = q_0$ for all $t\geq 0$. We also define a corresponding target-window $\gamma_{0,1} = \gamma_0$ = $(0, [0, \infty))$. An agent trajectory $\taua$ \emph{intercepts} target-window $\gamma_{i,j}$ if $\taua$ intercepts target $i$ at a time $t \in [\underline{t}_{i,j}, \overline{t}_{i,j}]$. 

$\mathcal{E}_\text{tw}$ contains an edge from target-window $\gamma_{i,j}$ to $\gamma_{i',j'}$ if and only if $i \neq i'$. This edge records the \emph{latest feasible departure time} from $\gamma_{i,j}$ to $\gamma_{i',j'}$, denoted by $\LFDT(\gamma_{i,j},\gamma_{i',j'})$. This value specifies the maximum $t \in [\underline{t}_{i,j}, \overline{t}_{i,j}]$ for which there exists a speed-admissible trajectory that intercepts $\gamma_{i',j'}$ after intercepting $\gamma_{i,j}$ at time $t$.\footnote{\cite{philip2025C} describes how to compute $\LFDT$.}

A \emph{tour} in $\mathcal{G}_\text{tw}$ is a path beginning and ending with $\gamma_0$, visiting at most one target-window per non-fictitious target, such that the sum of demands of visited targets is no larger than $d_\text{max}$. A \emph{partial tour} has the same definition as a tour, but does not need to end with $\gamma_0$. $\Gamma[k]$ denotes the $k$th element of a partial tour $\Gamma$, and $\Length(\Gamma)$ denotes the number of elements; we use the same notation for any \emph{sequence} also. An agent trajectory $\taua$ \emph{executes} a partial tour $\Gamma$ if $\taua$ intercepts the target-windows in $\Gamma$ in sequence. The optimal cost of a partial tour $\Gamma$, denoted as $c^*(\Gamma)$, is the minimum cost over all trajectories executing $\Gamma$. $c^*(\Gamma)$, and the associated trajectory, can be computed using the second-order cone program (SOCP) in Appendix \ref{sec:ptour_dist} in the supplementary material\footnote{This SOCP can be obtained by modifying the mixed-integer SOCP from \cite{philip2025Mixed}, which considers the case where the sequence of target-windows is not fixed.}.

We now formulate an ILP where a solution to the MT-VRP is described by a set of tours, $\mathcal{F}_\text{sol}$. For a tour $\Gamma$, let $\alpha(i, \Gamma) = 1$ if $\Gamma$ contains a target-window $\gamma_{i,j}$ of target $i$ and $\alpha(i, \Gamma) = 0$ otherwise. Let the set of all possible tours for the agents be $\mathcal{S} = \{\Gamma_1, \Gamma_2, \dots, \Gamma_{|\mathcal{S}|}\}$. For each $\Gamma_k \in \mathcal{S}$, define a binary variable $\theta_k$, which equals 1 if $\Gamma_k$ is chosen to be included in $\mathcal{F}_\text{sol}$ and 0 otherwise. Our ILP for the MT-VRP is as follows:

\begin{mini!}
{\{\theta_k\}_{\Gamma_k \in \mathcal{S}}}{\sum\limits_{\Gamma_k \in \mathcal{S}}c^*(\Gamma_k)\theta_k\label{eqn:ilp_objective}}{\label{optprob:mt_cvrp_ilp}}{}
\addConstraint{\sum\limits_{\Gamma_k \in \mathcal{S}}\theta_k \leq \nagt\label{eqn:ilp_agent_limit}}{}
\addConstraint{\sum\limits_{\Gamma_k \in \mathcal{S}}\alpha(i, \Gamma_k)\theta_k\geq 1 \;\; \forall i \in \{1, \dots, \ntar\}\label{eqn:ilp_visit_all_targets}}{}
\addConstraint{\theta_k \in \{0, 1\} \;\; \forall k \in \{1,\dots,|\mathcal{S}|\}\label{eqn:ilp_theta_nonneg_integer}}{}
\end{mini!}
\eqref{eqn:ilp_objective} minimizes the sum of costs of the selected tours. \eqref{eqn:ilp_agent_limit} ensures that we select up to $\nagt$ tours. \eqref{eqn:ilp_visit_all_targets} ensures that each target is visited by at least one tour. \eqref{eqn:ilp_theta_nonneg_integer} constrains $\theta_k$ to be binary. Capacity constraints are implicit in ILP \eqref{optprob:mt_cvrp_ilp}, since a tour must fully satisfy the demands of all targets it visits.

A convex relaxation of ILP \eqref{optprob:mt_cvrp_ilp} consists of \eqref{eqn:ilp_objective}-\eqref{eqn:ilp_visit_all_targets}, but replaces \eqref{eqn:ilp_theta_nonneg_integer} with the constraint $\theta_k \geq 0$; as in \cite{feillet2010tutorial}, we remove the upper limit on $\theta_k$ in the relaxation. An optimal solution $\theta^* = (\theta_1^*, \theta_2^*, \dots, \theta_{|\mathcal{S}|}^*)$ to ILP \eqref{optprob:mt_cvrp_ilp} is feasible for this relaxation, but not necessarily optimal: some fractional solution may be optimal instead. Thus, the branch-and-bound in Section \ref{sec:branch_and_bound} iteratively breaks ILP \eqref{optprob:mt_cvrp_ilp} into subproblems, each disallowing some set of edges $\mathcal{B}$ to eliminate fractional solutions. The aim is to generate some subproblem where $\theta^*$ is optimal for the relaxation. In particular, for a set $\mathcal{B} \subseteq \mathcal{E}_\text{tw}$, we define the subproblem ILP-$\mathcal{B}$ as ILP \eqref{optprob:mt_cvrp_ilp}, with the constraint that $\theta_k = 0$ for each $\Gamma_k$ traversing some $e \in \mathcal{B}$. Let LP-$\mathcal{B}$ be the convex relaxation of ILP-$\mathcal{B}$.

\section{Branch-and-Price with Relaxed Continuity}\label{sec:branch_and_bound}

We first provide an overview of our approach and then focus on the pricing problem. At the start of BPRC, we generate an initial feasible solution $\mathcal{F}_\text{inc}$ with cost $c_\text{inc}$, using the method from Section \ref{sec:gen_feas_tours}. We call $\mathcal{F}_\text{inc}$ the \emph{incumbent} and continually update it to be the best solution found so far. We also initialize a set of tours  $\mathcal{F}$, setting it equal to $\mathcal{F}_\text{inc}$; we continually add tours to $\mathcal{F}$ throughout the algorithm.

BPRC then constructs a branch-and-bound tree, where a node is a set $\mathcal{B} \subseteq \mathcal{E}_\text{tw}$ of disallowed edges. We expand nodes from a stack in a depth-first fashion. When expanding a node $\mathcal{B}$, we solve LP-$\mathcal{B}$ (traditionally called the \emph{master problem} in branch-and-price) using column generation. That is, when we initially solve LP-$\mathcal{B}$, we only optimize over $\theta_k$ variables for tours $\Gamma_k \in \mathcal{F}$, traversing no edge in $\mathcal{B}$, to keep the problem size tractable, assuming all other $\theta_k$ values are zero. This restriction of LP-$\mathcal{B}$ is known as the \emph{restricted master problem} (RMP). After solving the RMP using an off-the-shelf LP solver, we solve a \emph{pricing problem} to add tours to $\mathcal{F}$ that may improve the optimal cost of the RMP. Section \ref{sec:solving_pricing_problem} describes our method of solving the pricing problem. We alternate between the RMP and the pricing problem until the pricing problem adds no tours to $\mathcal{F}$, implying we have an optimal solution to LP-$\mathcal{B}$.

Whenever we find a new best integer solution during column generation, we update $\mathcal{F}_\text{inc}$ and $c_\text{inc}$. Additionally, the first time we solve the RMP for branch-and-bound node $\mathcal{B}$, the LP solver may return infeasible. This means that every feasible MT-VRP solution that can be assembled from the tours in $\mathcal{F}$ traverses some edge in $\mathcal{B}$. If so, we attempt to use the method from Section \ref{sec:gen_feas_tours} to produce a feasible MT-VRP solution $\mathcal{F}_\text{new}$ traversing no edge in $\mathcal{B}$, set $\mathcal{F} \leftarrow \mathcal{F} \cup \mathcal{F}_\text{new}$, then solve the RMP again. If we do not find any $\mathcal{F}_\text{new}$, LP-$\mathcal{B}$ is infeasible, and we discard the branch-and-bound node $\mathcal{B}$.

When we find an optimal solution $\theta$ to LP-$\mathcal{B}$, if $\theta$ is integer, $\mathcal{B}$ has no successors. If $\theta$ is fractional, we apply the ``conventional branching" rule from \cite{ozbaygin2017branch} as follows. First, we let the flow of an edge $e$ be the sum of $\theta_k$ values over all $\Gamma_k$ traversing $e$.
We select an edge $e$ not incident to $\gamma_0$ with minimum flow, then generate two successors to $\mathcal{B}$, where in the first successor, $e$ is disallowed, and in the second successor, $e$ is mandated by disallowing other edges appropriately (see \cite{ozbaygin2017branch}). This successor generation procedure ensures that an optimal integer solution to ILP-$\mathcal{B}$ is feasible for one successor, but the fractional solution $\theta$ is feasible in neither. Each successor $\mathcal{B}'$ is given a lower bound equal to the cost of $\theta$ for LP-$\mathcal{B}$. $\mathcal{B}'$ is pruned upon expansion if its lower bound is larger than $c_\text{inc}$.

\subsection{Pricing Problem}\label{sec:pricing_problem}
The pricing problem seeks a set of tours $\mathcal{F}_\text{price} = \{{}^1\Gamma, {}^2\Gamma, \dots, {}^{n_\text{found}}\Gamma\}$ such that the RMP on $\mathcal{F} \cup \mathcal{F}_\text{price}$ has a smaller optimal cost than the RMP on $\mathcal{F}$. $n_\text{found}$ is the number of tours found when solving the pricing problem and is not a user-specified parameter. We find $\mathcal{F}_\text{price}$ by leveraging the optimal dual solution to the RMP, denoted as $(\lambda_0, \lambda_1, \dots, \lambda_{\ntar})$, where $\lambda_0 \in \mathbb{R}_{\leq 0}$ is the dual variable corresponding to the constraint \eqref{eqn:ilp_agent_limit}, and  $\lambda_i \in \mathbb{R}_{\geq 0}$ is the dual variable corresponding to the $i^{th}$ constraint in \eqref{eqn:ilp_visit_all_targets}. Following \cite{feillet2010tutorial}, $\mathcal{F}_\text{price}$ can reduce the optimal cost of the RMP only if $\mathcal{F}_\text{price}$ contains a tour $\Gamma$ with negative \emph{reduced cost}, which is defined as follows:
\begin{align}
    \cstarred(\Gamma) = c^*(\Gamma) - \sum\limits_{i = 1}^{\ntar}\alpha(i, \Gamma)\lambda_i - \lambda_0.\label{eqn:reduced_cost_defn}
\end{align}
Thus the pricing problem seeks tours with negative reduced cost.\footnote{A tour $\Gamma$ already in $\mathcal{F}$ must have $\cstarred(\Gamma) \geq 0$, because otherwise $\lambda_0, \lambda_1, \dots, \lambda_\ntar$ would be dual infeasible for the RMP \cite{feillet2010tutorial}. However, in practice, due to floating-point arithmetic, we may compute a slightly negative $\cstarred(\Gamma)$. Thus, as in prior work \cite{kohl19992}, we seek tours with $\cstarred(\Gamma) < -\epsilon$ to avoid repeatedly adding tours to $\mathcal{F}$; in our implementation, $\epsilon = 10^{-4}$.}

Note that since the dual variables $\lambda_i \geq 0$ are subtracted in the reduced cost \eqref{eqn:reduced_cost_defn}, adding target-windows to a tour can decrease its reduced cost. This property is important in Section \ref{sec:solving_pricing_problem} for determining if a partial tour $\Gamma$ dominates another, $\Gamma'$. Specifically, our dominance criterion requires that any target-window appendable to $\Gamma'$ must also be appendable to $\Gamma$. This ensures that any decrease in reduced cost achievable by extending $\Gamma'$ is also achievable by extending $\Gamma$.


\subsection{BPRC Labeling Algorithm}\label{sec:solving_pricing_problem}
We solve the pricing problem using a labeling algorithm based on \cite{boland2006accelerated}, which forms the basis for state-of-the-art VRP pricing solvers. We first discuss the novel dominance criterion used in our labeling algorithm, then provide the algorithm itself.

For a partial tour $\Gamma$ to dominate another partial tour $\Gamma'$, two conditions must be met. First, any feasible extension of $\Gamma'$ (i.e., a sequence of target-windows that can be appended to it) must also be a feasible extension of $\Gamma$. Feasibility here means not violating capacity constraints or revisiting targets. Second, the minimum reduced cost over extensions of $\Gamma$ must be no greater than that from $\Gamma'$. \emph{Our primary contribution is a novel and computationally efficient method for verifying this second condition for the MT-VRP.}

We first explain the core idea before providing formal definitions. Assume the feasibility condition is met and that both $\Gamma$ and $\Gamma'$ end at the same target-window, $\gamma_{i,j}$. For an arbitrary $t' \in [\underline{t}_{i,j}, \overline{t}_{i,j}]$, let $\taua'$ be an optimal trajectory that executes $\Gamma'$ and intercepts $\gamma_{i,j}$ at space-time point $(q',t')$, where $q' = \tau_i(t')$. Suppose there exists a trajectory $\taua$ that executes $\Gamma$ and intercepts $\gamma_{i,j}$ at its starting point $(q, t)$. Let $\lambda = \sum\limits_{i = 1}^{\ntar}\alpha(i, \Gamma)\lambda_i + \lambda_0$, and $\lambda'$ be defined similarly for $\Gamma'$. Clearly, $\Gamma$ dominates $\Gamma'$ if for all $t'$, there exists a corresponding $\taua$ satisfying
\begin{align}
    c(\taua) + \|q'-q\|_2 -\lambda \leq c(\taua') - \lambda' \label{eq:red_cost_condition}.
\end{align}
This means that following $\taua$, then moving in a straight line from $q$ to $q'$, is no worse than following $\taua'$.
Checking this condition for every $t'$ is intractable, and thus we derive a sufficient condition for \eqref{eq:red_cost_condition} using easily computable bounds:

\begin{itemize}
    \item An \emph{upper bound} on $c(\taua) + \|q'-q\|_2$ is $c(\tau_\text{a,ub}) + \delta(\gamma_{i,j})$, where $\tau_\text{a,ub}$ is a trajectory that executes $\Gamma$ optimally under the constraint that $\tau_\text{a,ub}$ can only intercept target-windows at their start points, and $\delta(\gamma_{i,j})$ is the spatial length of $\gamma_{i,j}$. We call $\tau_\text{a,ub}$ an \emph{upper-bounding trajectory} for $\Gamma$. An example is shown in Fig. \ref{fig:segments_and_dominance_check} (a) and (c).
    \item A \emph{lower bound} on $c(\taua')$ is the cost of a trajectory $\tau_\text{a,lb}'$ that optimally executes $\Gamma'$, where $\tau_\text{a,lb}'$ is allowed to be discontinuous between its arrival and departure at any target-window. We call $\tau_\text{a,lb}'$ a \emph{lower-bounding trajectory} for $\Gamma'$. An example is shown in Fig. \ref{fig:segments_and_dominance_check} (b) and (c).
\end{itemize}
This yields the following sufficient condition for \eqref{eq:red_cost_condition}:
\begin{align}
    c(\tau_\text{a,ub}) + \delta(\gamma_{i,j}) -\lambda\leq c(\tau_\text{a,lb}') -\lambda'.\label{eqn:sufficient_red_cost_condition_full_window}
\end{align}
We constrain $\tau_\text{a,ub}$ to intercept target-windows at their start-points because this enables efficient computation of $c(\tau_\text{a,ub})$, as described in the subsequent section on the labeling algorithm. Otherwise, computing $c(\tau_\text{a,ub})$ would require solving an SOCP; we show in Section \ref{sec:numerical_results} that doing so for the vast number of partial tours in the pricing problem is computationally prohibitive. We relax continuity when computing $\tau_\text{a,lb}'$ similarly to avoid solving an SOCP.

Next, note that condition \eqref{eqn:sufficient_red_cost_condition_full_window} can be weak if the time window of $\gamma_{i,j}$ is large, thus limiting its ability to prune partial tours. To strengthen \eqref{eqn:sufficient_red_cost_condition_full_window}, we divide each target-window into segments and apply \eqref{eqn:sufficient_red_cost_condition_full_window} to each segment, as shown in Fig. \ref{fig:segments_and_dominance_check} (d). The new labels in BPRC are designed to store this extra information--namely, the upper and lower cost bounds and the accumulated dual variables for each partial tour. We now formally define the labels and the algorithm.



\subsubsection{BPRC Labels:}
We divide each target-window $\gamma_{i,j}$ into \emph{segments} $\xi_{i,j,1}, \dots, \xi_{i,j,n_\text{seg}(\gamma_{i,j})}$, where $\xi_{i,j,k} = (\gamma_{i,j}, [\underline{t}_{i,j,k}, \overline{t}_{i,j,k}])$ is the $k$th segment for $\gamma_{i,j}$, and $n_\text{seg}(\gamma_{i,j})$ is the number of segments for $\gamma_{i,j}$. The depot only gets one segment, i.e. $n_\text{seg}(\gamma_0) = 1$. We determine $n_\text{seg}(\gamma_{i,j})$ for $\gamma_{i,j} \neq \gamma_0$ via a user-specified parameter $n_\text{seg,tar}$, which is the number of segments per target. We allocate segments to target-windows so that each target-window gets at least one segment, and target-windows with longer time windows get more segments. Appendix \ref{sec:segment_allocation} gives the allocation formula.

Let $\delta(\xi_{i,j,k})$ be the spatial length of $\xi_{i,j,k}$, and let $s_{i,j,k} = (\tau_{i}(\underline{t}_{i,j,k}), \underline{t}_{i,j,k})$ be the starting point of $\xi_{i,j,k}$. We say an agent trajectory $\taua$ \emph{intercepts} segment $\xi_{i,j,k}$ if $\taua$ intercepts target $i$ in the time window of $\xi_{i,j,k}$.

\begin{figure}
    \centering
    \includegraphics[width=0.47\textwidth]{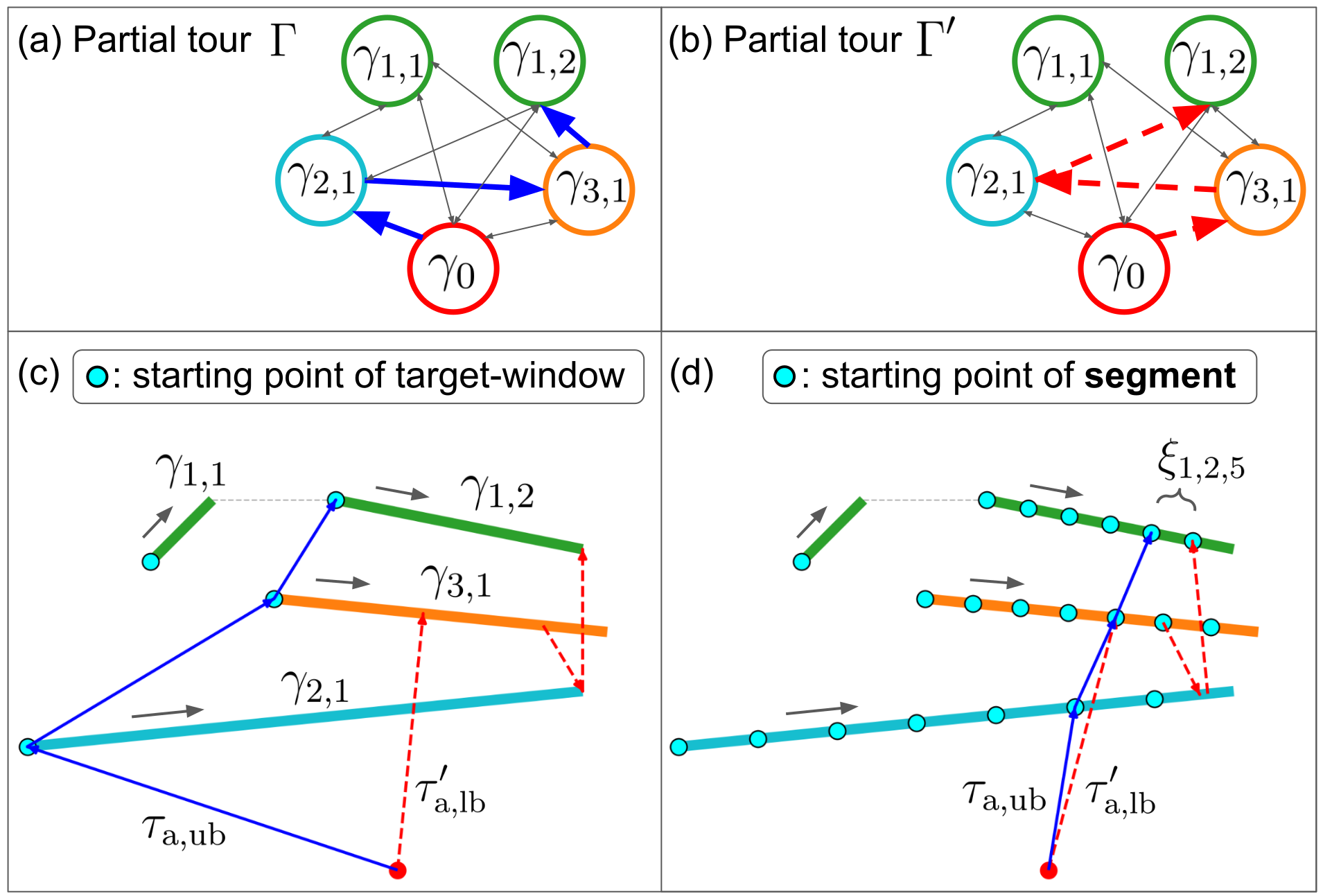}
    \caption{(a) and (b) show two partial tours $\Gamma$ and $\Gamma'$ ending with $\gamma_{1,2}$. We want to check if $\Gamma$ dominates $\Gamma'$. (c) We check dominance by comparing an upper-bounding trajectory $\tau_\text{a,ub}$ for $\Gamma$ and a lower-bounding trajectory $\tau_\text{a,lb}'$ for $\Gamma'$, as described in Section \ref{sec:solving_pricing_problem}. (d) We strengthen our dominance check by dividing $\gamma_{1,2}$ into segments and applying the check from (c) for each segment; an example is shown for the 5th segment of $\gamma_{1,2}$, denoted as $\xi_{1,2,5}$.}
    \label{fig:segments_and_dominance_check}
\end{figure}

Within the labeling algorithm, we represent a partial tour $\Gamma$ as a \emph{label}, denoted as $\text{Label}(\Gamma) = (\gamma_{i,j}, t, \sigma, \vec{b}, \vec{g}_\text{ub}, \vec{g}_\text{lb}, \lambda)$, whose elements we define below. Note that $\gamma_{i,j}, t, \sigma$, and $\vec{b}$ are elements typically stored in VRP labeling algorithms \cite{costa2019exact}, while $\vec{g}_\text{ub}, \vec{g}_\text{lb}$ and $\lambda$ are specific to our method for the MT-VRP. In particular,
\begin{itemize}
    \item $\gamma_{i,j}$ is the final target-window in $\Gamma$
    \item $t$ is the minimum time required to execute $\Gamma$
    \item $\sigma$ is the sum of demands of targets visited by $\Gamma$
    \item $\vec{b}$ is a binary vector with length $\ntar$, where $\vec{b}[i'] = 1$ if and only if $\Gamma$ visits target $i'$, $\sigma + d_{i'} > d_\text{max}$, or $t > \max\limits_{j' \in \{1, 2, \dots, n_\text{win}(i')\}}\hspace{-0.1cm}\LFDT(\gamma_{i,j}, \gamma_{i',j'})$. The last two conditions mean target $i'$ cannot be intercepted after intercepting $\gamma_{i,j}$ at time $t$, due to capacity or speed constraints.\footnote{Including both visitation and reachability status of targets in a label is a well-known strategy for strengthening dominance checks in VRP variants, introduced by \cite{feillet2004exact}.}
    \item $\vec{g}_\text{ub}$ is a vector of length $n_\text{seg}(\gamma_{i,j})$, where $\vec{g}_\text{ub}[k]$ is the cost of an upper-bounding trajectory for $\Gamma$ intercepting $\xi_{i,j,k}$ at its start point.
    \item $\vec{g}_\text{lb}$ is a vector of length $n_\text{seg}(\gamma_{i,j})$, where $\vec{g}_\text{lb}[k]$ is the cost of a lower-bounding trajectory for $\Gamma$ intercepting $\xi_{i,j,k}$.
    \item $\lambda = \sum\limits_{i = 1}^{\ntar}\alpha(i, \Gamma)\lambda_i + \lambda_0$.
\end{itemize}
Each label $l$ also stores a backpointer to a parent label. By traversing backpointers recursively, we can obtain the partial tour $\Gamma$ corresponding to $l$, denoted as \PartialTour($l$).

\subsubsection{BPRC Dominance Criteria:} Consider two labels $l = (\gamma_{i,j}, t, \sigma, \vec{b}, \vec{g}_\text{ub}, \vec{g}_\text{lb}, \lambda)$ and $l' = (\gamma_{i,j}, t', \sigma', \vec{b}', \pvec{g}_\text{ub}', \pvec{g}_\text{lb}', \lambda')$, both at target-window $\gamma_{i,j}$, and let $\Gamma = \PartialTour(l)$ and $\Gamma' = \PartialTour(l')$. We now define dominance so that if $l$ dominates $l'$, $\PartialTour(l)$ dominates $\PartialTour(l')$.   
\begin{itemize}
\item General case when $\gamma_{i,j} \neq \gamma_0$: $l$ dominates $l'$ if all the following conditions hold:
{\small
\begin{align}
    \sigma &\leq \sigma'\label{eqn:label_dominance_capacity}\\
    \vec{b}[i'] &\leq \vec{b}'[i'], \; ~\forall i' \in\{1,\dots,\ntar\}\label{eqn:label_dominance_visitation}\\
    \hspace{-0.1cm}\vec{g}_\text{ub}[k] + \delta(\xi_{i,j,k}) - \lambda &\leq \pvec{g}_\text{lb}'[k] - \lambda', \; \forall k \in \{1,\dots,n_\text{seg}(\gamma_{i,j})\}\label{eqn:label_dominance_cost}
\end{align}
}
\item Special case when $\gamma_{i,j} = \gamma_0$: $l$ dominates $l'$ if $\cstarred(\Gamma) \leq \cstarred(\Gamma')$.
\end{itemize}
\eqref{eqn:label_dominance_capacity}-\eqref{eqn:label_dominance_visitation} ensure that any target-windows appendable to $\Gamma$ are appendable to $\Gamma'$, and \eqref{eqn:label_dominance_cost} is a generalization of condition \eqref{eqn:sufficient_red_cost_condition_full_window}.

\subsubsection{BPRC Labeling Algorithm:} The algorithm (Alg. \ref{alg:label_setting}) first initializes a priority queue, where a label $l = (\gamma, t, \sigma, \vec{b}, \vec{g}_\text{ub}, \vec{g}_\text{lb}, \lambda)$ has higher priority than label $l' = (\gamma', t', \sigma', \vec{b}', \pvec{g}_\text{ub}', \pvec{g}_\text{lb}', \lambda')$ if $(t, \sigma, \min(\vec{g}_\text{lb}) - \lambda)$ is lexicographically smaller than $(t', \sigma', \min(\pvec{g}_\text{lb}') - \lambda')$. The priority queue is initialized with the label whose partial tour contains only the depot. Alg. \ref{alg:label_setting} then initializes a set $\mathfrak{D}[\gamma]$ of nondominated labels at each target-window $\gamma$ (Lines \ref{algline:start_init_labels}-\ref{algline:done_init_labels}). Alg. \ref{alg:label_setting}'s main loop performs a best-first search. Each iteration pops the highest-priority label $l = (\gamma_{i,j}, t, \sigma, \vec{b}, \vec{g}_\text{ub}, \vec{g}_\text{lb}, \lambda)$ from the priority queue, then iterates over the \emph{successor target-windows} of $l$, i.e. each target-window $\gamma_{i',j'}$ that satisfies
\begin{enumerate}
    \item $(\gamma_{i,j}, \gamma_{i',j'}) \in \mathcal{E}_\text{tw} \setminus \mathcal{B}$, preventing traversal of edges in $\mathcal{B}$

    \item $t \leq \LFDT(\gamma_{i,j}, \gamma_{i',j'})$, so intercepting $\gamma_{i',j'}$ after intercepting $\gamma_{i,j}$ at time $t$ does not violate the speed limit

    \item $\sigma + d_{i'} \leq d_\text{max}$, enforcing capacity constraints

    \item If $i' \neq 0$, then $\vec{b}[i'] = 0$, ensuring no targets are revisited
\end{enumerate}
For each successor target-window $\gamma_{i',j'}$, we create a successor label $l' = (\gamma_{i',j'}, t', \sigma', \vec{b}', \pvec{g}_\text{ub}', \pvec{g}_\text{lb}', \lambda')$ as follows:
\begin{itemize}
    \item $t'$ is the \emph{earliest feasible arrival time} from space-time point $(\tau_{i}(t), t)$ to $\gamma_{i',j'}$, denoted as $\EFAT(\tau_{i}(t), t, \gamma_{i',j'})$, which is the minimum time at which a speed-admissible trajectory can intercept $\gamma_{i',j'}$ after departing $(\tau_{i}(t), t)$. \cite{philip2025C} describes how to compute $t'$.
    \item $\sigma' = \sigma + d_{i'}$
    \item $\vec{b}'$ is identical to $\vec{b}$, except for the following modifications. First, if $i' \neq 0$, we set $\vec{b}'[i'] = 1$. Then, we set $\vec{b'}[i''] = 1$ for all targets $i''$ such that $\sigma' + d_{i''} > d_\text{max}$ or $t' > \max\limits_{j'' \in \{1, 2, \dots, n_\text{win}(i'')\}}\LFDT(\gamma_{i',j'}, \gamma_{i''j''})$.
    \item $\lambda' = \lambda$, if $i' = 0$, and $\lambda' = \lambda + \lambda_{i'}$ otherwise
    \item For each $k' \in \{1, 2, \dots, n_\text{seg}(\gamma_{i',j'})\}$, we compute $\pvec{g}'_\text{ub}[k']$ via Bellman's optimality principle:
    \begin{align}
        \hspace{-0.6cm}\pvec{g}'_\text{ub}[k'] = \hspace{-0.3cm}\min\limits_{k \in \{1, 2, \dots, n_\text{seg}(\gamma_{i,j})\}} \vec{g}_\text{ub}[k] + c_\text{start}(s_{i,j,k}, s_{i',j',k'})\label{eqn:segment_start_bellman_update}\hspace{-0.25cm}
    \end{align}
    where $c_\text{start}(s_{i,j,k}, s_{i',j',k'})$ is the Euclidean distance between the positions of $s_{i,j,k}$ and $s_{i',j',k'}$, if a speed-admissible trajectory exists from $s_{i,j,k}$ to $s_{i',j',k'}$, and $\infty$ otherwise. We compute $c_\text{start}$ for all pairs $(s_{i,j,k}, s_{i',j',k'})$ with $i \neq i'$ at the beginning of the BPRC algorithm.
    \item To compute each $\pvec{g}_\text{lb}'[k']$, we first check if $t' > \overline{t}_{i',j',k'}$; if so, $\pvec{g}_\text{lb}'[k'] = \infty$, since it is impossible for an agent to intercept $\xi_{i',j',k'}$ by executing $\PartialTour(l')$. If $t' \leq \overline{t}_{i',j',k'}$, we apply Bellman's principle, similarly to \eqref{eqn:segment_start_bellman_update}:
    \begin{align}
        \hspace{-0.6cm}\pvec{g}_\text{lb}'[k'] &= \hspace{-0.3cm}\min\limits_{k \in \{1, 2, \dots, n_\text{seg}(\gamma_{i,j})\}}\vec{g}_\text{lb}[k] + c_\text{seg}(\xi_{i,j,k}, \xi_{i',j',k'})\label{eqn:segment_bellman_update}\hspace{-0.25cm}
    \end{align}
    where $c_\text{seg}(\xi_{i,j,k}, \xi_{i',j',k'})$ is the minimum distance an agent must travel between an interception of $\xi_{i,j,k}$ and an interception of $\xi_{i',j',k'}$, where $\xi_{i',j',k'}$ is intercepted after $\xi_{i,j,k}$. We compute $c_\text{seg}$ for all pairs $(\xi_{i,j,k}, \xi_{i',j',k'})$ with $i \neq i'$ at the beginning of the BPRC algorithm, using the SOCP in Appendix \ref{sec:segment_socp}.
\end{itemize}
After generating the successor label $l'$, if $\gamma_{i',j'} = \gamma_0$ (i.e. $l'$ represents a tour $\Gamma'$), we perform additional pruning checks (Lines \ref{algline:check_prune_depot_label}-\ref{algline:prune_depot_label}). First, if any label $l''$ at $\gamma_0$ satisfies $\cstarred(\PartialTour(l'')) \leq \pvec{g}_\text{lb}'[1] - \lambda'$, we prune $l'$. This is because $\pvec{g}_\text{lb}'[1] - \lambda'$ is a lower bound on $\cstarred(\Gamma')$, and if this lower bound is not smaller than the reduced cost of a previously found tour, then $\cstarred(\Gamma')$ cannot be smaller either. Next, if $\pvec{g}_\text{lb}'[1] - \lambda' \geq -\epsilon$, we prune $l'$, since this means $\cstarred(\Gamma')$ is no smaller than $-\epsilon$. If both of these checks fail, we compute $\cstarred(\Gamma')$ via an SOCP and prune $l'$ if $\cstarred(\Gamma') \geq -\epsilon$. We perform the first two checks before the third to avoid the computational expense of solving an SOCP, if possible.

Next, if $l'$ has not yet been pruned, we check if $l'$ is dominated by any existing labels at $\gamma_{i',j'}$, and prune $l'$ if so (Line \ref{algline:prune_dominated_successor}). If $l'$ is not dominated, we prune labels stored at $\gamma_{i',j'}$ dominated by $l'$ and add $l'$ to $\mathfrak{D}[\gamma_{i',j'}]$ (Line \ref{algline:add_successor_to_nondominated_labels}). Then, if $\gamma_{i',j'} \neq \gamma_0$, we push $l'$ onto the priority queue, and if $\gamma_{i',j'} = \gamma_0$, we extract the tour corresponding to $l'$ and add it to the set $\mathcal{F}_\text{price}$ returned upon termination (Lines \ref{algline:push_onto_pqueue}-\ref{algline:add_to_Fprice}).

To speed up BPRC, we apply a heuristic from \cite{ozbaygin2017branch}, where each $\mathfrak{D}[\gamma_{i,j}]$ only stores the label at $\gamma_{i,j}$ with the smallest $\min(\vec{g}_\text{lb}) - \lambda$ so far. This modification may prevent finding tours with negative reduced cost. Thus, if this modified Alg. \ref{alg:label_setting} returns $\emptyset$, we run the unmodified Alg. \ref{alg:label_setting}.

\begin{algorithm}
\caption{SolvePricingProblem($\lambda, \mathcal{B}$)}\label{alg:label_setting}
\begin{algorithmic}[1]
\State PQUEUE = [$(\gamma_0, 0, 0, \underbrace{(0, 0, \dots, 0)}_{\ntar \text{ times}}, (0), (0), \lambda_0)$]\label{algline:label_setting_initialize_pqueue}
\State $\mathfrak{D}$ = dict()\label{algline:start_init_labels}
\For{$\gamma \in \mathcal{V}_\text{tw}$}
    $\mathfrak{D}[\gamma] = \{\}$\label{algline:done_init_labels}
\EndFor
\State $\mathcal{F}_\text{price} = \emptyset$
\While{\upshape PQUEUE is not empty}
    \State $l = $ PQUEUE.pop()
    \For{each $\gamma_{i',j'} \in \text{SuccessorTargetWindows}(l, \mathcal{B})$}\label{algline:pricing_iterate_over_successors}
        \State $l' = (\gamma_{i',j'}, t', \sigma', \vec{b}', \pvec{g}_\text{ub}', \pvec{g}_\text{lb}', \lambda') = \text{SuccessorLabel}(l, \gamma_{i',j'})$\label{algline:successor_label}
        \If{$\gamma_{i',j'} = \gamma_0$ AND CheckDepotPrune($l'$)}\label{algline:check_prune_depot_label}
            \State continue\label{algline:prune_depot_label}
        \EndIf
        \If{any $l'' \in \mathfrak{D}[\gamma_{i',j'}]$ dominates $l'$}\label{algline:check_successor_dominated}
            continue\label{algline:prune_dominated_successor}
        \EndIf
        \State Delete all $l'' \in \mathfrak{D}[\gamma_{i',j'}]$ dominated by $l'$\label{algline:prune_existing_labels}
        \State $\mathfrak{D}[\gamma_{i',j'}]$.add($l'$)\label{algline:add_successor_to_nondominated_labels}
        \If{$\gamma_{i',j'} \neq \gamma_0$}
            PQUEUE.push($l'$)\label{algline:push_onto_pqueue}
        \Else \;
            $\mathcal{F}_\text{price} = \mathcal{F}_\text{price} \cup \{\PartialTour(l')\}$\label{algline:add_to_Fprice}
        \EndIf
    \EndFor
\EndWhile

\State return $\mathcal{F}_\text{price}$
\end{algorithmic}
\end{algorithm}

\subsection{Generating Additional Feasible Solutions}\label{sec:gen_feas_tours}
To generate the initial incumbent and additional feasible solutions when the RMP is infeasible, we extend the feasible solution generation method from \cite{bhat2024AComplete}, which applies to the single-agent MT-TSP, to handle multiple agents and capacity constraints. In particular, \cite{bhat2024AComplete} expands partial tours from a stack in depth-first fashion. We instead expand sets of partial tours from a stack, where each set contains up to one partial tour per agent. Appendix \ref{sec:gen_feas_tours_full} fully describes the algorithm.

\subsection{Theoretical Analysis}
Branch-and-price finds an optimal solution to an ILP of type \eqref{optprob:mt_cvrp_ilp} if the dominance checks are correct, and the feasible solution generation method (Section \ref{sec:gen_feas_tours}) is complete.\footnote{Complete algorithms generate a feasible solution if one exists and report infeasible in finite time otherwise \cite{choset2005principles}.} We prove our dominance checks' correctness here; proofs of intermediate lemmas and completeness are in Appendix \ref{sec:additional_proofs}.
\begin{theorem}\label{thm:label_dominance_means_ptour_dominance}
Let $l = (\gamma_{i,j}, t, \sigma, \vec{b}, \vec{g}_\textnormal{ub}, \vec{g}_\textnormal{lb}, \lambda)$ and $l' = (\gamma_{i,j}, t', \sigma', \vec{b}', \pvec{g}_\textnormal{ub}', \pvec{g}_\textnormal{lb}', \lambda')$. If \eqref{eqn:label_dominance_capacity}-\eqref{eqn:label_dominance_cost} hold, then $\Gamma = \PartialTour(l)$ dominates $\Gamma' = \PartialTour(l')$.
\end{theorem}
\begin{proof}
Let $T'$ be a least-reduced-cost tour beginning with $\Gamma'$, and let $\lambda_\text{post} = \sum\limits_{i = 1}^{\ntar}\alpha(i, T')\lambda_i - \lambda_0 - \lambda'$. Let $\taua'$ be a least-cost, speed-admissible trajectory executing $T'$. Let $\tau_{\text{a,pre}}'$ be the portion of $\taua'$ ending at $\gamma_{i,j}$, and let $\tau_\text{a,post}'$ be the remaining portion. In Lemma \ref{lemma:cost_dominance_lemma} in Appendix \ref{sec:additional_proofs}, we prove that \eqref{eqn:label_dominance_cost} and the speed-admissibility of $\tau_\text{a,pre}'$ ensure that there is a trajectory $\tau_\text{a,pre}$ executing $\Gamma$ satisfying
\begin{align}
    c(\tau_\text{a,pre}) - \lambda \leq c(\tau_\text{a,pre}') - \lambda'.\label{eqn:taua_pre_vs_taua_pre_prime}
\end{align}
Let $\taua$ be the concatenation of $\tau_\text{a,pre}$ and $\tau_\text{a,post}'$. Adding $c(\tau_\text{a,post}')$ to both sides of \eqref{eqn:taua_pre_vs_taua_pre_prime}, we have
\begin{align}
    c(\tau_\text{a}) - \lambda \leq c(\tau_\text{a}') - \lambda'.\label{eqn:taua_vs_taua_prime}
\end{align}
Let $T'_\text{post}$ be the portion of $T'$ from $T'[\Length(\Gamma') + 1]$ onward, and let $\sigma_\text{post}'$ be its depleted capacity. Let $T$ be the concatenation of $\Gamma$ with $T'_\text{post}$. Since $T'$ is a tour, $T'_\text{post}$ does not revisit targets visited by $\Gamma'$. Combining this with \eqref{eqn:label_dominance_visitation}, $T'_\text{post}$ does not revisit targets visited by $\Gamma$ either, and thus $T$ does not revisit targets. Also, since $T'$ is a tour, its depleted capacity $\sigma' + \sigma_\text{post}'$ is no larger than $d_\text{max}$; combining this with \eqref{eqn:label_dominance_capacity}, the depleted capacity of $T$, i.e. $\sigma + \sigma_\text{post}'$, is no larger than $d_\text{max}$. Thus, since $T$ revisits no targets, satisfies capacity constraints, and begins and ends at $\gamma_0$, $T$ is a tour. 

Also, $\taua$ executes $T$, so $c^*(T) \leq c(\taua)$. This means
\begin{align}
    \cstarred(T) = c^*(T) - \lambda - \lambda_\text{post} &\leq c(\taua) - \lambda - \lambda_\text{post}\\
    &\leq c(\taua') - \lambda' - \lambda_\text{post}\label{eqn:cstarred_T_leq_c_taua_prime_minus_lambdas}\\
    &= \cstarred(T')
\end{align}
where \eqref{eqn:cstarred_T_leq_c_taua_prime_minus_lambdas} follows from \eqref{eqn:taua_vs_taua_prime}. Thus we have a tour $T$ beginning with $\Gamma$ with $\cstarred(T) \leq \cstarred(T')$, so $\Gamma$ dominates $\Gamma'$.
\end{proof}

\section{Numerical Results}\label{sec:numerical_results}
We ran experiments using an Intel i9-9820X 3.3GHz CPU with 128 GB RAM, with 10 cores and hyperthreading disabled. Our baseline, called \emph{Compact MICP}, is the MICP \cite{philip2025Mixed}, with additional decision variables and constraints based on \cite{desrochers1987vehicle} to enforce capacity constraints. In particular, we implemented (17)-(29) from \cite{philip2025Mixed}, then added an additional decision variable $\sigma_{i,j} \in [0, d_\text{max} - d_{i}]$ for each target-window $\gamma_{i,j}$, equal to the depleted capacity immediately before visiting $\gamma_{i,j}$, if $\gamma_{i,j}$ is visited. We then added a constraint 
{\small
\begin{align}
    \sigma_{i,j} + d_{i} \leq \sigma_{i',j'} + (d_\text{max} + d_i)(1 - y_{e}) \; \forall e = (\gamma_{i,j}, \gamma_{i',j'}) \in \mathcal{E}_\text{tw}
\end{align}
}
where $y_{e}$ is a binary decision variable from \cite{philip2025Mixed} which equals 1 if edge $e$ is traversed and 0 otherwise. We also implemented two ablations of BPRC, called \emph{BPRC-ablate-dominance} and \emph{BPRC-ablate-bounds}. BPRC-ablate-dominance is meant to show the benefit of our dominance checks at target-windows $\gamma \neq \gamma_0$ by removing these checks. In particular, BPRC-ablate-dominance does not store nondominated labels at any $\gamma \neq \gamma_0$ and does not store $\vec{g}_\text{ub}$ in its labels.\footnote{BPRC-ablate-dominance still stores $\vec{g}_\text{lb}$ to break ties on the priority queue and prune labels at the depot.} BPRC-ablate-bounds is meant to show the benefits of our upper and lower-bounding methods for computing $\vec{g}_\text{ub}$ and $\vec{g}_\text{lb}$, respectively. In particular, BPRC-ablate-bounds computes $\vec{g}_\text{ub}$ and $\vec{g}_\text{lb}$ exactly using SOCPs. We solved Compact MICP and all SOCPs with Gurobi 12.03.

We generated problem instances by extending the instance generation method from \cite{bhat2024AComplete} to multiple agents.
In all instances, each target had two time windows and a demand of 1. Then we varied the number of targets, the capacity, the number of agents, and the time window lengths in Experiments 1 to 4.
In these experiments, we set the number of segments per target, i.e. $n_\text{seg,tar}$, to 32 for BPRC, 4 for BPRC-ablate-bounds, and 8 for BPRC-ablate-dominance.\footnote{We tuned $n_\text{seg,tar}$ for BPRC to achieve the best median runtime over 10 instances with 20 targets, separate from the instances in the presented experiments. The ablations' median runtimes were equal to the time limit for each tested $n_\text{seg,tar}$ for these tuning instances, so we instead tuned the ablations on 10 instances with 15 targets.}
We varied $n_\text{seg,tar}$ in Experiment 5. We limited each planner's computation time to 10 min per instance. We also set a 100 GiB memory limit, which was only relevant for BPRC-ablate-dominance; the maximum measured memory usages for Compact MICP, BPRC-ablate-bounds, and BPRC were 1.41 GiB, 0.345 GiB, and 1.28 GiB, respectively. When we give runtime statistics (e.g. median) for BPRC-ablate-dominance, we only take statistics over instances where it did not reach the memory limit.

\begin{figure}
    \centering
    \includegraphics[width=0.47\textwidth]{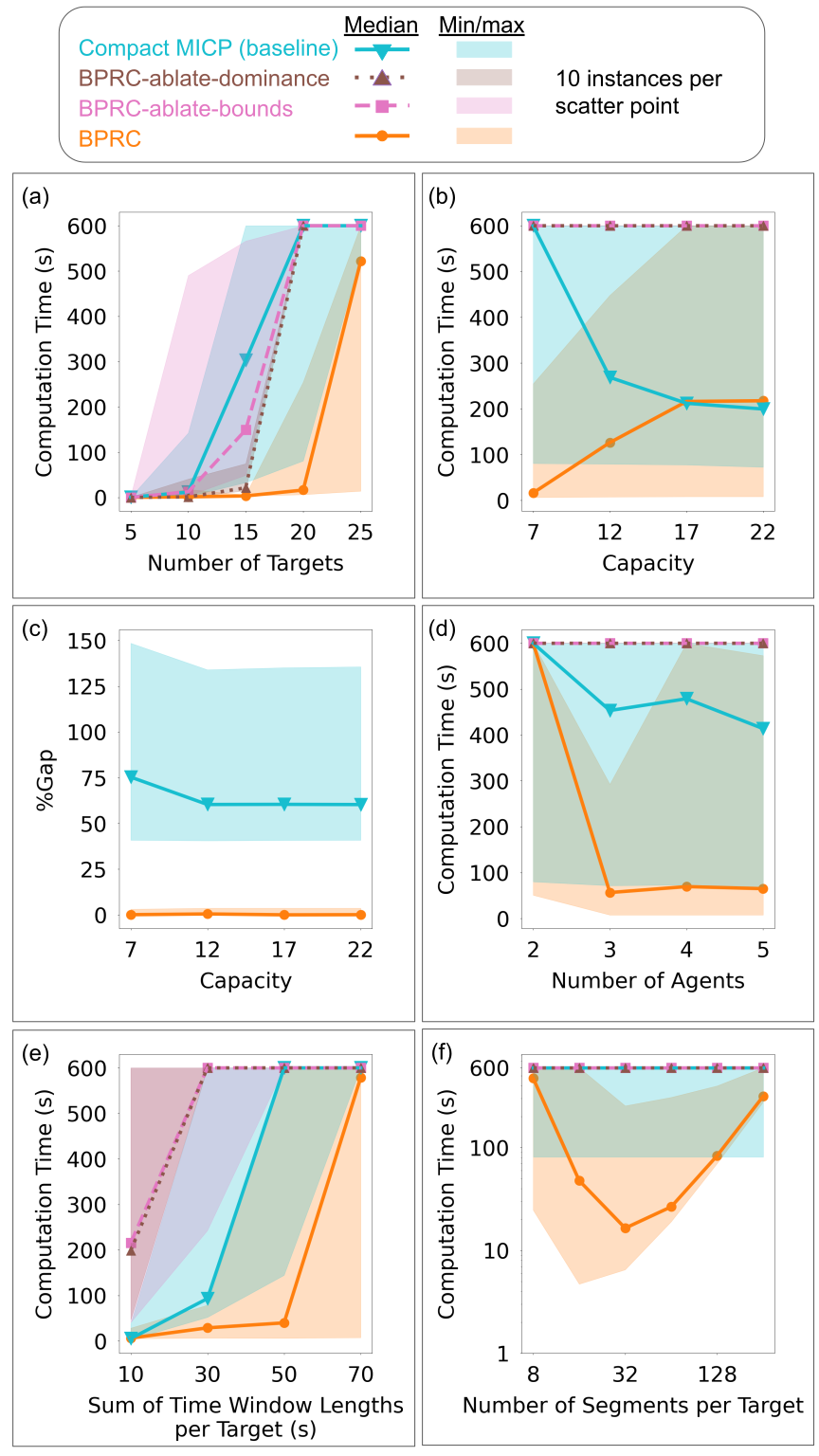}
    \caption{(a) Varying the number of targets. The gap in computation time between BPRC and the baseline and ablations widens as we increase the number of targets up to 20. The gap shrinks at 25 targets because BPRC reaches the time limit in 5 instances. We do not show a computation time for BPRC-ablate-dominance for 25 targets because it reached the memory limit in every instance. (b) Varying the capacity. BPRC shows advantage for small capacities. (c) \%Gap (as defined in Experiment 2) of each method's relaxation solution. BPRC's relaxation has a smaller \%Gap, and is thus tighter, in every instance. (d) Varying the number of agents. BPRC shows advantage when the number of agents is 3 or larger. (e) Varying the time window lengths. BPRC's performance advantage widens as we increase time window length up to 50, then shrinks because BPRC reaches the time limit in many instances. (f) Varying the number of segments in BPRC. Horizontal and vertical axes are log-scaled. MICP's min, median, and max are shown as flat lines without scatter points because the quantity being varied is specific to BPRC. We only show MICP's computation time to indicate that BPRC has smaller median computation time for a wide range of parameter values.}
    \label{fig:plots}
\end{figure}

\subsubsection{Experiment 1 - Varying the Number of Targets:}\label{sec:vary_num_targets}
We varied $\ntar$ from 5 to 25. Each instance had 3 agents with capacity $\lceil n_\text{tar}/\nagt\rceil$, i.e. this experiment considers instances where capacity is severely limited. Experiment 2 considers larger capacities. The sum of time window lengths per target was 50. The results are in Fig. \ref{fig:plots} (a). For 15 to 20 targets, BPRC's median runtime is more than an order of magnitude smaller than Compact MICP's. BPRC has smaller runtime than the ablations as well, indicating that our dominance checks, and our cheap methods of evaluating them, are beneficial. For 25 targets, BPRC reaches the time limit in 5 instances and thus has similar median runtime to the baseline and ablations. For 25 targets, pricing was the bottleneck, taking more than 80\% of the runtime in 8 of the 10 instances.

\subsubsection{Experiment 2 - Varying the Capacity:}\label{sec:vary_capacity}
We varied the capacity from 7 to 22, fixing the number of targets to 20, the number of agents to 3, and the sum of time window lengths per target to 50. The results are in Fig. \ref{fig:plots} (b). BPRC has smaller median runtimes than Compact MICP for small capacities, and slightly larger median runtimes for large capacities. The ablations perform poorly for all capacities.

As stated in Section \ref{sec:cvrp_related_work}, one reason that branch-and-price is favored over compact formulations for VRPs is that the relaxation of branch-and-price's integer program tends to be tighter than the relaxations of compact formulations. To validate that this holds in the MT-VRP, for each instance in Experiment 2, for both BPRC and Compact MICP, we computed a lower bound $c_\text{LB}$ on the optimal cost by solving the convex relaxation of the method's integer or mixed-integer program.\footnote{For BPRC, $c_\text{LB}$ is the optimal cost of LP-$\mathcal{B}$ at the root branch-and-bound node, i.e. computing it does not require enumerating $\mathcal{S}$.} We then computed the gap of the optimal cost from each method's lower bound, defined as $\%\text{Gap} = (c_\text{opt} - c_\text{LB})/c_\text{LB}*100\%.$ As shown in Fig. \ref{fig:plots} (c), BPRC's relaxation is empirically tighter than Compact MICP's.\footnote{We only take statistics over instances where at least one method found the optimum, i.e. 8 instances for capacity 17 and 9 instances for capacity 22.}

\subsubsection{Experiment 3 - Varying the Number of Agents:}\label{sec:vary_num_agents}
We varied the number of agents from 2 to 5. Each instance had 20 targets, capacity 10, and sum of time window lengths per target equal to 50.
Fig. \ref{fig:plots} (d) shows the results. BPRC's median computation time drops steeply as we increase the number of agents from 2 to 3, then remains roughly constant as we increase the number of agents from 3 to 5. For 3 to 5 agents, BPRC's median computation time is significantly lower than the baseline's and the ablations'.

The large runtime for 2 agents appears to be because the agent limit \eqref{eqn:ilp_agent_limit} is directly affecting the optimal cost when $\nagt = 2$, but not when $\nagt > 3$. That is, in 4 of the 5 instances where some method found the optimum for $\nagt = 2$, increasing $\nagt$ to 3 actually resulted in a different optimal solution with a smaller cost. However, increasing $\nagt$ from 3 to 4 did not change the optimal solution in any instance. Increasing $\nagt$ can only decrease the optimal cost if $\lambda_0$ is strictly negative, which is what we often observed in the 2-agent instances: the median $\lambda_0$ value was -52, and the largest was -47231. Meanwhile, among all instances with $\nagt > 2$, the median $\lambda_0$ was 0, and the largest was only -889. Prior work \cite{afsar2021vehicle} has observed that large $\lambda_0$ values can slow the convergence of branch-and-price, which is what we observed here for $\nagt = 2$.

\subsubsection{Experiment 4 - Varying the Time Window Lengths:}\label{sec:vary_time_window_lengths}
We varied the sum of time window lengths per target from 10 to 70. Each instance had 20 targets and 3 agents with capacity 7.
The results are in Fig. \ref{fig:plots} (e). All methods' runtimes tend to increase with time window length. BPRC's median runtime is smaller than Compact MICP's for lengths 30 and 50, and smaller than the ablations' for lengths 10, 30, and 50.

\subsubsection{Experiment 5 - Varying the Number of Segments:}\label{sec:vary_num_segments}
We varied $n_\text{seg,tar}$ from 8 to 256, on a log-scale, i.e. we tested values 8, 16, 32, etc. Each instance had 20 targets, 3 agents with capacity 7, and sum of time window lengths per target equal to 50. The results are in Fig. \ref{fig:plots} (f). BPRC's median runtime is smaller than the other methods' minimum runtime for $n_\text{seg,tar}$ from 16 to 64 (inclusive). Thus, BPRC shows advantage even when we change $n_\text{seg,tar}$ by a factor of 4.

BPRC's runtime is large for large $n_\text{seg,tar}$ because computing distances between all pairs of segments becomes expensive. BPRC's runtime is large for small $n_\text{seg,tar}$ because dominance checks become weaker, causing more retained labels per target-window, making pricing expensive.

\section{Conclusion}
In this paper, we introduced BPRC, a new algorithm to find the optimum for the MT-VRP. We demonstrated that it finds the optimum with less computation time than a baseline, particularly when the agents have small capacities. As mentioned in Section \ref{sec:numerical_results}, pricing is a bottleneck in BPRC, and thus a direction for future work is to employ techniques such as the \textit{ng}-path relaxation \cite{baldacci2011new} to reduce time spent in the pricing problem.

\section*{Acknowledgments}
This material is partially based on work supported by the National Science Foundation (NSF) under Grant No. 2120219 and 2120529. Any opinions, findings, conclusions, or recommendations expressed in this material are those of the author(s) and do not necessarily reflect views of the NSF.

\bigskip

\bibliography{aaai2026}

\clearpage

\appendix
\section{SOCP to Compute Partial Tour Cost}\label{sec:ptour_dist}
In this section, for a target-window $\gamma_{i,j}$, let $\Tar(\gamma_{i,j}) = i$, $\underline{t}(\gamma_{i,j}) = \underline{t}_{i,j}$, and $\overline{t}(\gamma_{i,j}) = \overline{t}_{i,j}$. For partial tour $\Gamma = (\Gamma[1], \Gamma[2], \dots, \Gamma[\Length(\Gamma)])$, we compute $c^*(\Gamma)$ as the optimal cost of the following SOCP:
\begin{mini!}
{\{q_k, t_k, l_{\Delta,k},l_{\text{dist},k}\}_{k = 2}^{\Length(\Gamma)}}{\sum\limits_{k = 2}^{\Length(\Gamma)}l_{\text{dist},k}\label{eqn:ptour_dist_objective}}{\label{optprob:ptour_dist}}{}
\addConstraint{\underline{t}(\Gamma[k]) \leq t_k \leq \overline{t}(\Gamma[k]) \; \forall k\label{eqn:ptour_dist_time_window}}{}
\addConstraint{l_{\Delta,k} = q_{k} - q_{k - 1} \; \forall k\label{eqn:ptour_dist_displacement_consistency}}{}
\addConstraint{q_k = \tau_{\Tar(\Gamma[k])}(t_k) \; \forall k\label{eqn:ptour_dist_position_consistency}}{}
\addConstraint{l_{\text{dist},k} \leq v_\text{max}(t_k - t_{k - 1}) \; \forall k\label{eqn:ptour_dist_speed_limit}}{}
\addConstraint{\|l_{\Delta,k}\| \leq l_{\text{dist},k} \; \forall k\label{eqn:ptour_dist_soc}}{}
\end{mini!}
The decision variables of problem \eqref{optprob:ptour_dist}, for each $k \in \{2, 3, \dots, \Length(\Gamma)\}$ are the interception position $q_k \in \mathbb{R}^2$ and time $t_k$ at $\Gamma[k]$, the displacement $l_{\Delta,k} \in \mathbb{R}^2$ from $q_{k - 1}$ to $q_k$, and the distance $l_\text{dist}$ from $q_{k - 1}$ to $q_k$. Since $\Gamma[1] = \gamma_0$, there are no decision variables for $k = 1$; $q_0$ is already defined as the depot position, and we define $t_0 = 0$. The objective \eqref{eqn:ptour_dist_objective} minimizes the distance traveled. Constraint \eqref{eqn:ptour_dist_time_window} ensures that target-windows are intercepted within their time windows. \eqref{eqn:ptour_dist_displacement_consistency} ensures that the displacement variables equal the displacements between their corresponding position variables. \eqref{eqn:ptour_dist_position_consistency} ensures that the interception positions match the corresponding targets' positions at the interception times. \eqref{eqn:ptour_dist_speed_limit} enforces the speed limit, and \eqref{eqn:ptour_dist_soc} ensures that at the optimum, the distance variables are the norms of their corresponding displacement variables.

\section{SOCP to Compute Distance Between Segments}\label{sec:segment_socp}
We compute $c_\text{seg}(\xi_{i,j,k}, \xi_{i',j',k'})$ for segments $\xi_{i,j,k}$ and $\xi_{i',j',k'}$ as the optimal cost of the following SOCP:
\begin{mini!}
{q,t,q',t',l_\Delta,l_\text{dist}}{l_\text{dist} \label{eqn:segment_dist_objective}}{\label{optprob:segment_dist}}{}
\addConstraint{\underline{t}_{i,j,k} \leq t \leq \overline{t}_{i,j,k}\label{eqn:segment_dist_start_time_window}}{}
\addConstraint{\underline{t}_{i',j',k'} \leq t' \leq \overline{t}_{i',j',k'}\label{eqn:segment_dist_end_time_window}}{}
\addConstraint{l_\Delta = q' - q\label{eqn:segment_dist_displacement_consistency}}{}
\addConstraint{q = \tau_i(t)\label{eqn:segment_dist_start_position_consistency}}{}
\addConstraint{q' = \tau_{i'}(t')\label{eqn:segment_dist_end_position_consistency}}{}
\addConstraint{l_\text{dist} \leq v_\text{max}(t' - t)\label{eqn:segment_dist_speed_limit}}{}
\addConstraint{\|l_\Delta\| \leq l_\text{dist}\label{eqn:segment_dist_soc}}{}
\end{mini!}
The decision variables of problem \eqref{optprob:segment_dist} are the departure position $q \in \mathbb{R}^2$ and time $t$ from $\xi_{i,j,k}$, the arrival position $q' \in \mathbb{R}^2$ and time $t'$ at $\xi_{i',j',k'}$, the displacement $l_\Delta \in \mathbb{R}^2$ from $q$ to $q'$, and the distance $l_\text{dist}$ from $q$ to $q'$. The objective \eqref{eqn:segment_dist_objective} minimizes the distance traveled. Constraints \eqref{eqn:segment_dist_start_time_window}-\eqref{eqn:segment_dist_end_time_window} ensure that departure and arrival times lie in their respective time windows. \eqref{eqn:segment_dist_displacement_consistency} ensures that the displacement variable equals the displacement between $q$ and $q'$. \eqref{eqn:segment_dist_start_position_consistency} and \eqref{eqn:segment_dist_end_position_consistency} ensure that the departure and arrival positions match the corresponding targets' positions at the departure and arrival times. \eqref{eqn:segment_dist_speed_limit} enforces the speed limit, and \eqref{eqn:segment_dist_soc} ensures that at the optimum, $l_\text{dist}$ equals the norm of $l_\Delta$. If problem \eqref{optprob:segment_dist} is infeasible, then $c_\text{seg}(\xi_{i,j,k}, \xi_{i',j',k'} = \infty)$.

\section{Allocating Segments to Target-Windows}\label{sec:segment_allocation}
This section describes how we allocate segments to non-depot target-windows. For each $\gamma_{i,j}$, except the $\gamma_{i,j}$ with the longest time window for each target $i$, we set:
\begin{align}
    n_\text{seg}(\gamma_{i,j}) = \max\left(1, \left\lfloor\frac{\overline{t}_{i,j} - \underline{t}_{i,j}}{\sum\limits_{j' = 1}^{n_\text{win}(i)}(\overline{t}_{i,j'} - \underline{t}_{i,j'})} n_\text{seg,tar}\right\rfloor\right)\label{eqn:num_segs_per_tw}
\end{align}
\eqref{eqn:num_segs_per_tw} ensures that every target-window gets at least one segment, and that target-windows with longer time windows get more segments. Finally, for $\gamma_{i,j}$ with the longest time window for target $i$, we allocate all remaining segments.

\section{Generating Additional Feasible Solutions}\label{sec:gen_feas_tours_full}
This section describes the function used to generate the initial incumbent and generate additional feasible solutions when the RMP is infeasible in Section \ref{sec:branch_and_bound}. Given a set of disallowed edges $\mathcal{B} \subseteq \mathcal{E}_\text{tw}$, GenerateFeasibleSolution attempts to generate a set of tours solving the MT-VRP while traversing no edge in $\mathcal{B}$. We do so using the DFS in Alg. \ref{alg:gen_feas_tours}, which attempts to quickly find tours without regard to cost. Alg. \ref{alg:gen_feas_tours} takes as input a set of disallowed edges $\mathcal{B} \subseteq \mathcal{E}_\text{tw}$, which is $\emptyset$ when generating the initial incumbent, and equal to a branch-and-bound node otherwise. Alg. \ref{alg:gen_feas_tours} maintains a stack of \emph{multi-agent-labels}, or \emph{m-labels}, for short; in subsequent text, we will always refer to an m-label as an m-label, and never as a label, to avoid confusion with the labels in Section \ref{sec:solving_pricing_problem}. An m-label $u = (\gamma, t, \sigma, \vec{b}, \nptour)$ represents a set of partial tours $\{\Gamma^1, \Gamma^2, \dots, \Gamma^{\nptour}\}$, assigned to agents $1, 2, \dots, \nptour$, respectively, where $\Gamma^1, \Gamma^2, \dots, \Gamma^{\nptour - 1}$ are tours, and $\Gamma^\nptour$ is not a tour. $u$ represents this set of partial tours as follows:
\begin{itemize}
    \item $\gamma$ is the final target-window in $\Gamma^{\nptour}$
    \item $t$ is the interception time at $\gamma$
    \item $\sigma$ is the sum of demands of targets visited by $\Gamma^{\nptour}$
    \item $\vec{b}$ is a binary vector with length $\ntar$ where $\vec{b}[i] = 1$ if and only if target $i$ is visited by some tour in $\{\Gamma^1, \Gamma^2, \dots, \Gamma^\nptour\}$
\end{itemize}
An m-label additionally stores a backpointer to a parent m-label, and thus we can extract the set of partial tours from a m-label via backpointer traversal. In particular, by traversing backpointers, we will obtain a sequence of target-windows
\begin{align}
    (&\gamma_0, {}^2\gamma^1, \dots, {}^{n^1}\gamma^1, \gamma_0, {}^2\gamma^2, \dots, {}^{n^2}\gamma^2,\nonumber\\
     &\dots, \gamma_0, {}^2\gamma^\nptour, \dots, {}^{n^\nptour}\gamma^\nptour)
\end{align}
where $n^k - 1$ is the number of target-windows not equal to $\gamma_0$ between the $k$th and $(k + 1)$th occurrence of $\gamma_0$, if the $(k + 1)$th occurrence exists, and the number of target-windows after the $k$th occurrence of $\gamma_0$ otherwise. We then let $\Gamma^k = (\gamma_0, {}^2\gamma^k, \dots, {}^{n^k}\gamma^k, \gamma_0)$ for each agent $k < \nptour$, and let $\Gamma^k = (\gamma_0, {}^2\gamma^k, \dots, {}^{n^k}\gamma^k)$ for $k = \nptour$.

In the context of Alg. \ref{alg:gen_feas_tours}, m-label $u = (\gamma, t, \sigma, \vec{b}, \nptour)$ \emph{dominates} $u' = (\gamma, t', \sigma', \vec{b}, \nptour)$ if $t \leq t'$ and $\sigma \leq \sigma'$. Alg. \ref{alg:gen_feas_tours} maintains a dictionary $\mathfrak{D}$, where a key is $(\gamma, \vec{b}, \nptour)$ for some target-window $\gamma$, binary vector $\vec{b}$, and partial tour index $\nptour$, and the value for a key is a set of mutually nondominated m-labels containing $\gamma, \vec{b}$ and $\nptour$. Whenever Alg. \ref{alg:gen_feas_tours} accesses $\mathfrak{D}[(\gamma, \vec{b}, \nptour)]$ for a key $(\gamma, \vec{b}, \nptour)$ not in $\mathfrak{D}$, we assume $\mathfrak{D}[(\gamma, \vec{b}, \nptour)] = \emptyset$.

When Alg. \ref{alg:gen_feas_tours} expands an m-label $u = (\gamma_{i,j}, t, \sigma, \vec{b}, \nptour)$ in its main loop, it prunes $u$ if $u$ is dominated by some m-label in $\mathfrak{D}$ (Line \ref{algline:gen_feas_tours_dominate_expand}). If $u$ is not pruned, and if $u$ represents a set of partial tours visiting all targets, we perform backpointer traversal to obtain a set of partial tours $\{\Gamma^1, \Gamma^2, \dots, \Gamma^\nptour\}$, append $\gamma_0$ to $\Gamma^\nptour$ to make it a tour (recall that the others are already tours), and return the tours (Line \ref{algline:gen_feas_tours_return}). Otherwise, we iterate over the successor m-labels of $u$, defined as follows.

First, if $\nptour < \nagt$, $u$ represents a set of fewer than $\nagt$ partial tours, and we have a successor m-label $u' = (\gamma_0, 0, 0, \vec{b}, \nptour + 1)$ corresponding to the beginning of a new partial tour. Additionally, we generate a successor m-label corresponding to each \emph{successor target-window} of $u$, where a target-window $\gamma_{i',j'}$ is a successor target-window of $u$ if all of the following hold:
\begin{enumerate}
    \item $(\gamma_{i,j}, \gamma_{i',j'}) \in \mathcal{E}_\text{tw}\setminus \mathcal{B}$, which prevents traversal of disallowed edges
    \item $t \leq \LFDT(\gamma_{i,j}, \gamma_{i',j'})$, which ensures that it is feasible to intercept $\gamma_{i',j'}$ after intercepting $\gamma_{i,j}$ at time $t$
    \item $\sigma + d_{i'} \leq d_\text{max}$, which enforces capacity constraints
    \item $\vec{b}[i'] = 0$, which prevents targets from being revisited
    \item If $n = \nagt$, for all targets $i'' \neq i'$ such that $\vec{b}[i''] = 0$, we have $\EFAT(\tau_{i}(t), t, \gamma_{i',j'}) \leq \max\limits_{j'' \in \{1, 2, \dots, n_\textnormal{win}(i'')\}}\LFDT(\gamma_{i',j'}, \gamma_{i'',j''})$. This ensures that visiting $\gamma_{i',j'}$ does not prevent agent $n$ from visiting some yet unvisited target.
\end{enumerate}

The successor m-label of $u$ corresponding to successor target-window $\gamma_{i',j'}$ is $u' = (\gamma_{i',j'}, t', \sigma', \vec{b}', \nptour)$, where
\begin{enumerate}
    \item $t' = \EFAT(\tau_{i}(t), t, \gamma_{i',j'})$
    \item $\sigma' = \sigma + d_{i'}$
    \item $\vec{b}'$ is the same as $\vec{b}$, but with $\vec{b}[i'] = 1$
\end{enumerate}
For each successor m-label $u'$, we check if $u'$ is dominated by a m-label in $\mathfrak{D}$, and prune $u'$ if so (Line \ref{algline:gen_feas_tours_dominate_before_push}). Otherwise, we add $u'$ to $\mathfrak{D}$ and push $u'$ onto the stack.

\begin{algorithm}
\caption{GenerateFeasibleSolution($\mathcal{B}$)}\label{alg:gen_feas_tours}
\begin{algorithmic}[1]

\State STACK = [($\gamma_0, 0, 0, \underbrace{(0, 0, \dots, 0)}_{\ntar \text{ times}}, 1$)]\label{algline:gen_feas_tours_initialize_stack}
\State $\mathfrak{D}$ = dict()
\While{STACK is not empty}\label{algline:gen_feas_tours_while_loop}
    \State $u = (\gamma_{i,j}, t, \sigma, \vec{b}, \nptour) = $ STACK.pop()\label{algline:gen_feas_tours_pop}
    \If{any $u' \in \mathfrak{D}[(\gamma_{i,j}, \vec{b}, \nptour)]$ dominates $u$}\label{algline:gen_feas_tours_check_dominate_expand}
        \State continue\label{algline:gen_feas_tours_dominate_expand}
    \EndIf
    \If{$\text{sum}(\vec{b}) = \ntar$}
        return ReconstructTours($u$)\label{algline:gen_feas_tours_return}
    \EndIf
    \If{$\nptour < \nagt$}
        STACK.push($(\gamma_0, 0, 0, \vec{b}, \nptour + 1)$)\label{algline:new_partial_tour}
    \EndIf
    \For{$u' = (\gamma_{i',j'}, t', \sigma', \vec{b}', \nptour)$ in\\\phantom{for }SuccessorMLabels($u, \mathcal{B}$)}
        \If{any $u'' \in \mathfrak{D}[(\gamma_{i',j'}, \vec{b}', \nptour')]$ dominates $u'$}
            \State continue\label{algline:gen_feas_tours_dominate_before_push}
        \EndIf
        \State $\mathfrak{D}[(\gamma_{i',j'}, \vec{b}', \nptour')]$.add($u'$)
        \State STACK.push($u'$)
    \EndFor
\EndWhile
return $\emptyset$\label{algline:gen_feas_tours_return_emptyset}
\end{algorithmic}
\end{algorithm}

\section{Additional Proofs}\label{sec:additional_proofs}
In this section, we state and prove the lemmas required for Thm. \ref{thm:label_dominance_means_ptour_dominance} in the main paper. We also state and prove the completeness of Alg. \ref{alg:gen_feas_tours}, which generates an initial feasible solution at a branch-and-bound node $\mathcal{B}$ when the RMP is infeasible.
\begin{lemma}\label{lemma:segment_start_graph_upper_bound_start}
Let $l$ be a label at target-window $\gamma_{i,j}$, and let $\Gamma = \PartialTour(l)$. Let $\taua$ be a least-cost, speed-admissible trajectory that executes $\Gamma$ and intercepts segment $\xi_{i,j,k}$ of $\gamma_{i,j}$ at time $\underline{t}_{i,j,k}$. $\vec{g}_\textnormal{ub}[k]$ is an upper bound on $c(\taua)$.
\end{lemma}
\begin{proof}
We proceed by induction.

\underline{Base Case:} Let $l$ be the label used to initialize the priority queue on Alg. \ref{alg:label_setting}, Line \ref{algline:label_setting_initialize_pqueue}. $\Gamma = \PartialTour(l) = (\gamma_0)$, and the least-cost trajectory $\taua$ executing $\Gamma$ has cost 0, since $\taua$ stays at the depot. Thus $\vec{g}_\text{ub}[1] = 0$ upper-bounds $c(\taua)$.

\underline{Induction Hypothesis:} Let $l = (\gamma_{i,j}, t, \sigma, \vec{b}, \vec{g}_\text{ub}, \vec{g}_\text{lb}, \lambda)$, and let $l' = (\gamma_{i',j'}, t', \sigma', \vec{b}', \pvec{g}'_\text{ub}, \pvec{g}'_\text{lb}, \lambda')$ be a successor of $l$ defined as in Section \ref{sec:pricing_problem}. Suppose Lemma \ref{lemma:segment_start_graph_upper_bound_start} holds for $l$.

\underline{Induction Step:} Let $k'$ be fixed and arbitrary. First, consider the case where $\pvec{g}'_\text{ub}[k'] = \infty$. Then $\pvec{g}'_\text{ub}[k']$ is a trivial upper bound and we are done. Now consider the case where $\pvec{g}'_\text{ub}[k'] \neq \infty$. Let
\begin{align}
    k = \hspace{-0.25cm}\argmin\limits_{k'' \in \{1, 2, \dots, n_\text{seg}(\gamma_{i,j})\}} \vec{g}_\text{ub}[k''] + c_\text{start}(s_{i,j,k''}, s_{i',j',k'})\label{eqn:k_argmin}
\end{align}
$\pvec{g}'_\text{ub}[k'] \neq \infty$ implies $\vec{g}_\text{ub}[k] \neq \infty$, so the induction hypothesis ensures that there is a speed-admissible trajectory $\bar{\tau}_\text{a}$ executing $\Gamma = \PartialTour(l)$ and intercepting $\xi_{i,j,k}$ at time $\underline{t}_{i,j,k}$ with $c(\bar{\tau}_\text{a}) \leq \vec{g}_\text{ub}[k]$. Consider the trajectory $\bar{\tau}_\text{a}'$ that follows $\bar{\tau}_\text{a}$, then travels in a straight line from $s_{i,j,k}$ to $s_{i',j',k'}$. $\pvec{g}'[k'] \neq \infty$ implies $c_\text{start}(s_{i,j,k}, s_{i',j',k'}) \neq \infty$. The speed-admissibility of $\bar{\tau}_\text{a}$ and $c_\text{start}(s_{i,j,k}, s_{i',j',k'}) \neq \infty$ imply that $\bar{\tau}_\text{a}'$ is speed-admissible. Furthermore,
\begin{align}
    c(\bar{\tau}'_\text{a}) &= c(\bar{\tau}_\text{a}) + c_\text{start}(s_{i,j,k}, s_{i',j',k'})\\
    &\leq \vec{g}_\text{ub}[k] + c_\text{start}(s_{i,j,k}, s_{i',j',k'})\label{eqn:segment_start_induction_step}.
\end{align}
Since $\bar{\tau}'_\text{a}$ is speed-admissible, executes $\Gamma'$, and intercepts $\xi_{i',j',k'}$, $c(\bar{\tau}_\text{a}')$ upper-bounds the cost of a least-cost trajectory with the same properties; thus, so does RHS of \eqref{eqn:segment_start_induction_step}, which equals $\pvec{g}'_\text{ub}[k']$ by \eqref{eqn:k_argmin} and \eqref{eqn:segment_start_bellman_update}.
\end{proof}

\begin{lemma}\label{lemma:segment_start_graph_upper_bound}
Let $l$ be a label at target-window $\gamma_{i,j}$, and let $\Gamma = \PartialTour(l)$. Let $\taua$ be a least-cost, speed-admissible trajectory that executes $\Gamma$ and intercepts segment $\xi_{i,j,k}$ of $\gamma_{i,j}$ at time $t_\textnormal{f}$. $\vec{g}_\textnormal{ub}[k] + \delta(\xi_{i,j,k})$ is an upper bound on $c(\taua)$.
\end{lemma}
\begin{proof}
If $\vec{g}_\text{ub}[k] = \infty$. $\vec{g}_\textnormal{ub}[k] + \delta(\xi_{i,j,k}) = \infty$ is a trivial upper bound and we are done. Now suppose $\vec{g}_\text{ub}[k] \neq \infty$. By Lemma \ref{lemma:segment_start_graph_upper_bound_start}, there is a trajectory $\bar{\tau}_\text{a}$ executing $\Gamma$ and intercepting $\xi_{i,j,k}$ at $\underline{t}_{i,j,k}$ with $c(\bar{\tau}_\text{a}) \leq \vec{g}_\textnormal{ub}[k]$. Let $\bar{\tau}_\text{a}'$ be the trajectory that follows $\bar{\tau}_\text{a}$, then follows $\tau_i$ from time $\underline{t}_{i,j,k}$ to time $t_\text{f}$. $\bar{\tau}_\text{a}'$ is speed-admissible, since $\bar{\tau}_\text{a}$ is speed-admissible and since we assumed $\tau_i$ moves no faster than $v_\text{max}$ in Section \ref{sec:problem_setup}. $c(\bar{\tau}_\text{a}') \leq c(\bar{\tau}_\text{a}) + \delta(\xi_{i,j,k})$, since $\tau_i$ cannot be followed along $\xi_{i,j,k}$ for distance greater than $\delta(\xi_{i,j,k})$. Since $c(\bar{\tau}_\text{a}') \leq c(\bar{\tau}_\text{a}) + \delta(\xi_{i,j,k})$, and since $c(\bar{\tau}_\text{a}) \leq \vec{g}_\textnormal{ub}[k]$ by Lemma \ref{lemma:segment_start_graph_upper_bound_start},
\begin{align}
    c(\bar{\tau}_\text{a}') \leq \vec{g}_\textnormal{ub}[k] + \delta(\xi_{i,j,k})\label{eqn:upper_bound_on_feas_trj}
\end{align}
Since $\bar{\tau}_\text{a}'$ is speed-admissible, executes $\Gamma$, and intercepts $\xi_{i,j,k}$ at time $t_\text{f}$, $c(\bar{\tau}_\text{a}')$ upper-bounds the cost of a least-cost trajectory with the same properties, namely $\taua$. Thus the RHS of \eqref{eqn:upper_bound_on_feas_trj} is an upper bound as well.
\end{proof}

\begin{lemma}\label{lemma:min_execution_time}
For each label $l = (\gamma_{i,j}, t, \sigma, \vec{b}, \vec{g}_\textnormal{ub}, \vec{g}_\textnormal{lb}, \lambda)$ generated in Alg. \ref{alg:label_setting}, $t$ is the minimum execution time of $\Gamma = \PartialTour(l)$.
\end{lemma}
\begin{proof}
We use induction to show that for all possible partial tour lengths $n \in \{1, 2, \dots, \ntar + 2\}$, Lemma \ref{lemma:min_execution_time} holds for labels $l$ satisfying $\Length(\PartialTour(l)) = n$.

\underline{Base Case:} $(n = 1)$ The initial label $l$ on Line \ref{algline:label_setting_initialize_pqueue} represents $\Gamma = (\gamma_0)$, the only partial tour with length 1, which has minimum execution time 0. The $t$ value for the $l$ is 0, so Lemma \ref{lemma:min_execution_time} holds for $l$.

\underline{Induction Hypothesis:} $(1 < n < \ntar + 2)$ Suppose Lemma \ref{lemma:min_execution_time} holds for all labels representing partial tours with length $n - 1$.

\underline{Induction Step:} Let $l' = (\gamma_{i',j'}, t', \sigma', \vec{b}', \pvec{g}'_\textnormal{ub}, \pvec{g}'_\textnormal{lb}, \lambda')$ be a fixed and arbitrary label among the labels representing partial tours with length $n$, and let $\Gamma' = \PartialTour(l')$. Let $l = (\gamma_{i,j}, t, \sigma, \vec{b}, \vec{g}_\textnormal{ub}, \vec{g}_\textnormal{lb}, \lambda)$ be the predecessor of $l'$, and let $\Gamma = \PartialTour(l)$. $\Length(\Gamma) = n - 1$, so Lemma \ref{lemma:min_execution_time} holds for $l$.

Suppose for the sake of contradiction that the minimum execution time of $\Gamma'$ is $t^* < t'$. Let $\taua^*$ be a speed-admissible trajectory that executes $\Gamma'$ in time $t^*$. Let $t^{**}$ be the time at which $\taua^*$ intercepts $\gamma_{i,j}$.

Since Lemma \ref{lemma:min_execution_time} holds for $l$ by the induction hypothesis, $t$ is the minimum execution time of $\Gamma$. Combining this with the fact that $\taua^*$ executes $\Gamma$ in time $t^{**}$, we have $t \leq t^{**}$. Let $\taua'$ be the speed-admissible trajectory that begins at $(\tau_i(t), t)$, follows $\tau_i$ from $t$ to $t^{**}$, then follows $\taua^*$ from $t^{**}$ to $t^*$. $\taua'$ is speed-admissible while following $\tau_i$ from $t$ to $t^{**}$ because we assumed in Section \ref{sec:problem_setup} that $\tau_i$ moves no faster than $v_\text{max}$, and $\taua'$ is speed-admissible from $t^{**}$ onwards since $\taua^*$ is speed-admissible. $\taua'$ is thus a speed-admissible trajectory beginning at $(\tau_i(t), t)$ and intercepting $\gamma_{i',j'}$ at time $t^*$, so $\EFAT(\tau_i(t), t, \gamma_{i',j'}) \leq t^*$. The LHS of this inequality equals $t'$, which is a contradiction.
\end{proof}

\begin{lemma}\label{lemma:segment_graph_lower_bound}
Let $l$ be a label at $\gamma_{i,j}$, and let $\Gamma = \PartialTour(l)$. Let $\taua$ be a speed-admissible trajectory that executes $\Gamma$ and intercepts segment $\xi_{i,j,k}$ of $\gamma_{i,j}$. $\vec{g}_\textnormal{lb}[k]$ is a lower bound on $c(\taua)$.
\end{lemma}
\begin{proof}
We proceed by induction.

\underline{Base Case:} Let $l$ be the label used to initialize the priority queue on Alg. \ref{alg:label_setting}, Line \ref{algline:label_setting_initialize_pqueue}. $\vec{g}_\text{lb}[1] = 0$, which lower-bounds the cost of any trajectory.

\underline{Induction Hypothesis:} Let $l = (\gamma_{i,j}, t, \sigma, \vec{b}, \vec{g}_\text{ub}, \vec{g}_\text{lb}, \lambda)$, and let $l' = (\gamma_{i',j'}, t', \sigma', \vec{b}', \pvec{g}'_\text{ub}, \pvec{g}'_\text{lb}, \lambda')$ be a successor of $l$ defined as in Section \ref{sec:pricing_problem}. Suppose Lemma \ref{lemma:segment_graph_lower_bound} holds for $l$.

\underline{Induction Step:} Let $k'$ be fixed and arbitrary. Let $\Gamma' = \PartialTour(l')$. Let $\taua'$ be a speed-admissible trajectory that executes $\Gamma'$ and intercepts $\xi_{i',j',k'}$. Let $\xi_{i,j,k}$ be the segment of $\gamma_{i,j}$ intercepted by $\taua'$. Let $\tau'_\text{a,pre}$ be the portion of $\taua'$ that terminates upon interception of $\xi_{i,j,k}$, and let $\tau'_\text{a,post}$ be the remaining portion. $\vec{g}_\text{lb}[k] \leq c(\tau'_\text{a,pre})$ by the induction hypothesis, and $c_\text{seg}(\xi_{i,j,k}, \xi_{i',j',k'}) \leq c(\tau'_\text{a,post})$ by definition of $c_\text{seg}$. Thus
\begin{align}
    \vec{g}_\text{lb}[k] + c_\text{seg}(\xi_{i,j,k}, \xi_{i',j',k'}) \leq c(\tau'_\text{a,pre}) + c(\tau'_\text{a,post})\label{eqn:segment_lower_bound}
\end{align}
As described in Section \ref{sec:solving_pricing_problem}, when computing $\pvec{g}_\text{lb}'[k']$, we first check if $t' > \overline{t}_{i',j',k'}$, and if so, we set $\pvec{g}_\text{lb}'[k'] = \infty$. Since $\taua'$ intercepts $\xi_{i',j',k'}$, $\taua'$ must execute $\Gamma'$ in time $\tilde{t} \leq \overline{t}_{i',j',k'}$. Since $t'$ is the minimum execution time of $\Gamma'$ by Lemma \ref{lemma:min_execution_time}, we have $t' \leq \tilde{t}$, and thus $t' \leq \overline{t}_{i',j',k'}$, so the check $t' > \overline{t}_{i',j',k'}$ fails.

Upon the failure of this first check, the min operator in \eqref{eqn:segment_bellman_update} selects $\pvec{g}_\text{lb}'[k']$ no larger than the LHS of \eqref{eqn:segment_lower_bound}, and the RHS of \eqref{eqn:segment_lower_bound} is $c(\taua')$, so $\pvec{g}_\text{lb}'[k'] \leq c(\taua')$.
\end{proof}

\begin{lemma}\label{lemma:cost_dominance_lemma}
Let $l = (\gamma_{i,j}, t, \sigma, \vec{b}, \vec{g}_\textnormal{ub}, \vec{g}_\textnormal{lb}, \lambda)$ and $l' = (\gamma_{i,j}, t', \sigma', \vec{b}', \pvec{g}_\textnormal{ub}', \pvec{g}_\textnormal{lb}', \lambda')$. Let $t_\textnormal{f} \in [\underline{t}_{i,j}, \overline{t}_{i,j}]$. Let $\taua'$ be a least-cost trajectory executing $\PartialTour(l')$ and intercepting $\gamma_{i,j}$ at time $t_\textnormal{f}$. If \eqref{eqn:label_dominance_cost} holds, then there exists a speed-admissible trajectory $\taua$ executing $\Gamma = \PartialTour(l)$ and intercepting $\gamma_{i,j}$ at time $t_\textnormal{f}$ such that $c(\taua) - \lambda \leq c(\taua') - \lambda'$.
\end{lemma}
\begin{proof}
Let $\xi_{i,j,k}$ be a segment whose time window contains $t_\text{f}$. By \eqref{eqn:label_dominance_cost} and Lemma \ref{lemma:segment_graph_lower_bound},
\begin{align}
\hspace{-0.19cm}\vec{g}_\text{ub}[k] + \delta(\xi_{i,j,k}) - \lambda \leq \pvec{g}_\text{lb}'[k] - \lambda' \leq c(\tau_\text{a,pre}') - \lambda'.\label{eqn:apply_segment_graph_lower_bound}
\end{align}
The speed-admissibility of $\tau_\text{a,pre}'$ ensures that the RHS of \eqref{eqn:apply_segment_graph_lower_bound} is finite, and thus the LHS is finite as well\footnote{The LHS is not $-\infty$ because $\vec{g}_\text{ub}[k]$ is a sum of distances, and $\delta(\xi_{i,j,k})$ is a distance, both which which are nonnegative, and $\lambda \neq \infty$. We know $\lambda \neq \infty$ because if it was, the RMP must have been dual unbounded, i.e. primal infeasible, and we would have discarded the branch-and-bound node associated with this pricing problem.}. The finite LHS, combined with Lemma \ref{lemma:segment_start_graph_upper_bound}, ensures that a speed-admissible trajectory $\tau_\text{a,pre}$ exists that executes $\Gamma$ and intercepts $\gamma_{i,j}$ at time $t_\text{f}$ such that $c(\tau_\text{a,pre}) \leq \vec{g}_\text{ub}[k] + \delta(\xi_{i,j,k})$. Combining this with \eqref{eqn:apply_segment_graph_lower_bound}, we have $c(\tau_\text{a,pre}) - \lambda \leq c(\tau_\text{a}') - \lambda'$.
\end{proof}

In the following completeness theorem, we say an MT-VRP solution $\mathcal{F}_\text{sol}$ is $\mathcal{B}$-feasible if it traverses no edge in $\mathcal{B}$.
\begin{theorem}\label{thm:gen_feas_tours_completeness}
(Completeness of initial feasible solution generation) For a set $\mathcal{B} \subseteq \mathcal{E}_\text{tw}$, if no $\mathcal{B}$-feasible solutions exist, Alg. \ref{alg:gen_feas_tours} reports infeasible in finite time by returning $\emptyset$. If some $\mathcal{B}$-feasible solution exists, Alg. \ref{alg:gen_feas_tours} returns such a solution.
\end{theorem}
\begin{proof}
First, suppose no $\mathcal{B}$-feasible solutions exist. Then condition 1 in the definition of a successor target-window for Alg. \ref{alg:gen_feas_tours} prevents Alg. \ref{alg:gen_feas_tours} from ever executing Line \ref{algline:gen_feas_tours_return}, which returns a $\mathcal{B}$-feasible solution. Since Alg. \ref{alg:gen_feas_tours} implicitly searches a finite tree of m-labels, it must terminate. Since Alg. \ref{alg:gen_feas_tours} cannot terminate on Line \ref{algline:gen_feas_tours_return}, it must terminate on Line \ref{algline:gen_feas_tours_return_emptyset}, which returns $\emptyset$.

Next, suppose there is some $\mathcal{B}$-feasible solution. As stated above, Alg. \ref{alg:gen_feas_tours} terminates, so we must now prove that Alg. \ref{alg:gen_feas_tours} terminates on Line \ref{algline:gen_feas_tours_return}, which returns some $\mathcal{B}$-feasible solution, rather than Line \ref{algline:gen_feas_tours_return_emptyset}. To prove that Alg. \ref{alg:gen_feas_tours} terminates on Line \ref{algline:gen_feas_tours_return}, we now show that the stack is never empty on Line \ref{algline:gen_feas_tours_while_loop}, which prevents termination on Line \ref{algline:gen_feas_tours_return_emptyset}.

Let $\mathcal{F}_\text{sol}$ be some $\mathcal{B}$-feasible solution with cardinality at least 1.\footnote{If there is at least 1 target, $|\mathcal{F}_\text{sol}| \geq 1$. Even if there are 0 targets, the solution containing a single tour that stays at the depot is $\mathcal{B}$-feasible, and has cardinality 1. Thus, if any $\mathcal{B}$-feasible solutions exist, some $\mathcal{B}$-feasible solution exists with cardinality at least 1.} Let $P_\text{sol} = (\Gamma^1, \Gamma^2, \dots, \Gamma^{|\mathcal{F}_\text{sol}|})$ be some permutation of $\mathcal{F}_\text{sol}$; note that while $\mathcal{F}_\text{sol}$ is a set of tours, $P_\text{sol}$ is a sequence of tours. We say that an m-label $(\gamma_{i,j}, t, \sigma, \vec{b}, n)$ is \emph{substitutable} into $P_\text{sol}$ if the following \emph{substitution conditions} are satisfied:
\begin{enumerate}
    \item $n \leq |\mathcal{F}_\text{sol}|$.\label{cond:substitution1}
    
    \item $\gamma_{i,j}$ is an element of $\Gamma^n$.\label{cond:substitution2}

    \item There is a speed-admissible agent trajectory executing $\Gamma^n$ that intercepts $\gamma_{i,j}$ no earlier than $t$.\label{cond:substitution3}

    \item $\sigma$ is no larger than the sum of demands of targets visited by $\Gamma^n$ at or before the first occurrence of $\gamma_{i,j}$.\label{cond:substitution4}
    
    \item For all targets $i^\dagger$ such that $\vec{b}[i^\dagger] = 1$, $i^\dagger$ is either visited by $\Gamma^n$ at or before $\gamma_{i,j}$, or visited by some $\Gamma^j$ with $j < n$.\label{cond:substitution5}
\end{enumerate}
The intuition behind the substitution conditions is as follows. By traversing the backpointers of m-label $u = (\gamma_{i,j}, t, \sigma, \vec{b}, n)$, we will obtain a sequence of partial tours $P_\text{sol}' = (\Gamma^{'1}, \Gamma^{'2}, \dots, \Gamma^{'n})$, where $\Gamma^{'n}$ is not a tour and the others are tours. Let $l = \Length(\Gamma^{'n})$, and let $\Gamma^n[l + 1:]$ be the suffix of $\Gamma^n$ beginning at index $l + 1$. Let $\Gamma^{''n}$ be the tour obtained by concatenating $\Gamma^{'n}$ with $\Gamma^n[l + 1:]$, where $\Gamma^{'n}$ comes first. If the substitution conditions are satisfied, we can construct a $\mathcal{B}$-feasible solution $\mathcal{F}_\text{sol}'' = \{\Gamma^{'1}, \Gamma^{'2}, \dots, \Gamma^{''n}, \Gamma^{n + 1}, \dots, \Gamma^{|\mathcal{F}_\text{sol}|}\}$, where we have essentially substituted partial tours associated with $u$ into $\mathcal{F}_\text{sol}$.

We now show that every time the condition on Line \ref{algline:gen_feas_tours_while_loop} is checked, the stack contains an m-label $(\gamma_{i, j}, t, \sigma, \vec{b}, n)$ substitutable into $P_\text{sol}$. Let $\Delta$ be the set of all m-labels that (i) have been popped from the stack, (ii) were not dominated on Line \ref{algline:gen_feas_tours_check_dominate_expand}, and (iii) are substitutable into $P_\text{sol}$. We have two cases:

\underline{Case 1:} suppose $\Delta = \emptyset$. In this case, $u = (\gamma_0, 0, 0, \underbrace{(0, 0, \dots, 0)}_{\ntar \text{ times}}, 1)$ must be on the stack. $u$ is substitutable into $P_\text{sol}$, since all of the following hold:
\begin{itemize}
    \item $1 \leq |\mathcal{F}_\text{sol}|$, satisfying substitution condition \ref{cond:substitution1}.

    \item $\gamma_0$ is an element of $\Gamma^1$, satisfying substitution condition \ref{cond:substitution2}.

    \item Some speed-admissible trajectory must execute $\Gamma^1$ since $\mathcal{F}_\text{sol}$ is a feasible MT-VRP solution, and all such trajectories intercept $\gamma_0$ at time $0$. This satisfies substitution condition \ref{cond:substitution3}.

    \item The sum of demands of targets visited by $\Gamma^1$ at or before the first occurrence of $\gamma^0$ is 0, satisfying substitution condition \ref{cond:substitution4}.

    \item No targets $i^\dagger$ have $\vec{b}[i^\dagger] = 1$, satisfying substitution condition \ref{cond:substitution5}.
\end{itemize}

\underline{Case 2:} suppose $\Delta \neq \emptyset$. In this case, let $u = (\gamma_{i,j}, t, \sigma, \vec{b}, n)$ be the element of $\Delta$ with largest possible $n$ and with $\gamma_{i,j}$ that has largest possible index in $\Gamma^{n}$.

\underline{Case 2.1:} suppose $\gamma_{i,j}$ is the second-to-last element of $\Gamma^{n}$. Suppose for the sake of contradiction that $n = |\mathcal{F}_\text{sol}|$. Since $u$ is substitutable into $P_\text{sol}$, $\vec{b}$ is all ones, meaning we returned on Line \ref{algline:gen_feas_tours_return} after popping $u$. This is a contradiction, because then we would not be checking Line \ref{algline:gen_feas_tours_while_loop} anymore. Thus, $n < |\mathcal{F}_\text{sol}|$. This means that when we expanded $u$, Line \ref{algline:new_partial_tour} pushed m-label $u' = (\gamma_0, 0, 0, \vec{b}, n + 1)$ onto the stack. $u'$ is substitutable into $P_\text{sol}$ since all of the following hold:
\begin{itemize}
    \item $n < |\mathcal{F}_\text{sol}|$ implies $n + 1 \leq |\mathcal{F}_\text{sol}|$, satisfying substitution condition \ref{cond:substitution1}.

    \item $\gamma_0$ is an element of $\Gamma^{n + 1}$, satisfying substitution condition \ref{cond:substitution2}.

    \item Some speed-admissible trajectory must execute $\Gamma^{n + 1}$ since $\mathcal{F}_\text{sol}$ is a feasible MT-VRP solution, and all such trajectories intercept $\gamma_0$ at time $0$. This satisfies substitution condition \ref{cond:substitution3}.

    \item The sum of demands of targets visited by $\Gamma^{n + 1}$ at or before the first occurrence of $\gamma^0$ is 0, satisfying substitution condition \ref{cond:substitution4}.

    \item Since $u$ is substitutable into $P_\text{sol}$, $\vec{b}$ satisfies substitution condition \ref{cond:substitution5}.
\end{itemize}

\underline{Case 2.2:} suppose $\gamma_{i, j}$ is not the second-to-last element of $\Gamma^{n}$. Let $\gamma_{i',j'}$ be the target-window that immediately follows $\gamma_{i, j}$ in $\Gamma^{n}$. $\gamma_{i',j'}$ must be a successor target-window of $u$, since all of the following hold:
\begin{itemize}
    \item $\Gamma^{n}$ traverses no edge in $\mathcal{B}$, by the $\mathcal{B}$-feasibility of $\mathcal{F}_\text{sol}$, and $\Gamma^{n}$ traverses edge $(\gamma_{i,j}, \gamma_{i',j'})$ by definition of $\gamma_{i',j'}$, so $(\gamma_{i,j}, \gamma_{i',j'}) \in \mathcal{E}_\text{tw} \setminus \mathcal{B}$. This satisfies condition 1 in the definition of a successor target-window.

    \item Since $u$ is substitutable into $P_\text{sol}$, there is a speed-admissible trajectory $\taua$ executing $\Gamma^{n}$ intercepting $\gamma_{i,j}$ no earlier than $t$ (substitution condition \ref{cond:substitution3}). This means $t \leq \LFDT(\gamma_{i,j}, \gamma_{i',j'})$, since $\taua$ departs $\gamma_{i,j}$ no earlier than $t$, then intercepts $\gamma_{i',j'}$. Thus condition 2 in the definition of a successor target-window is satisfied.

    \item Let $\sigma_\text{dep}$ be the capacity depleted by $\Gamma^{n}$ upon visiting $\gamma_{i,j}$. Since $u$ is substitutable into $P_\text{sol}$, $\sigma \leq \sigma_\text{dep}$ (substitution condition \ref{cond:substitution4}), so $\sigma + d_{i'} \leq \sigma_\text{dep} + d_{i'}$. The RHS is the capacity depleted by $\Gamma^n$ upon visiting $\gamma_{i',j'}$, and since $\Gamma^n$ is a tour, the RHS is no larger than $d_\text{max}$. Thus $\sigma + d_{i'} \leq d_\text{max}$, so condition 3 in the definition of a successor target-window is satisfied.

    \item Since $\mathcal{F}_\text{sol}$ is a feasible MT-VRP solution where $\Gamma^n$ visits target $i'$, $\Gamma^j$ does not visit target $i'$ for all $j < n$. Since $u$ is substitutable into $P_\text{sol}$, and since target $i'$ is visited after $\gamma_{i,j}$ in $P_\text{sol}$, we have $\vec{b}[i'] = 0$ (substitution condition \ref{cond:substitution5}), satisfying condition 4 in the definition of a successor target-window.

    \item Suppose $n = \nagt$. Since $u$ is substitutable into $P_\text{sol}$, there is a speed-admissible trajectory $\taua$ executing $\Gamma^n$ and intercepting $\gamma_{i,j}$ at time no earlier than $t$ (substitution condition \ref{cond:substitution3}). Let $t^\#$ be the time at which $\taua$ intercepts $\gamma_{i',j'}$. The speed-admissibility of $\taua$ ensures that $\EFAT(\tau_i(t), t, \gamma_{i',j'}) \leq t^\#$. Furthermore, since each target $i''$ not yet intercepted at time $t^\#$ by $\taua$ is intercepted by $\taua$ at some future time, $t^\# \leq \max\limits_{j'' \in \{1, 2, \dots, n_\text{win}(i'')\}}\LFDT(\gamma_{i',j'}, \gamma_{i'',j''})$ for all such targets $i''$. This satisfies condition 5 in the definition of a successor target-window.
\end{itemize}
Let $u' = (\gamma_{i',j'}, t', \sigma', \vec{b'}, n)$ be m-label associated with the successor target-window $\gamma_{i',j'}$ when expanding $u$.

\underline{Case 2.2.1:} Suppose $u'$ was dominated by some $u'' = (\gamma_{i',j'}, t'', \sigma'', \vec{b}', n)$ in $\mathfrak{D}[\gamma_{i',j'}, \vec{b}', n)]$ on Line \ref{algline:gen_feas_tours_dominate_before_push}. Without loss of generality, let $u''$ be an element of $\mathfrak{D}[(\gamma_{i',j'}, \vec{b}', n)]$ that is not dominated by any other element of $\mathfrak{D}[(\gamma_{i',j'}, \vec{b}', n)]$. $u'$ and $u''$ both are substitutable into $P_\text{sol}$, since all of the following hold:
\begin{itemize}
    \item Since $u$ is substitutable into $P_\text{sol}$, $n \leq |\mathcal{F}_\text{sol}|$, so both $u'$ and $u''$ satisfy substitution condition \ref{cond:substitution1}.

    \item $\gamma_{i',j'}$ is an element of $\Gamma^n$ by definition of $\gamma_{i',j'}$, so $u'$ and $u''$ both satisfy substitution condition \ref{cond:substitution2}.

    \item Since $u$ is substitutable into $P_\text{sol}$, there is a speed-admissible trajectory $\taua$ executing $\Gamma^n$ and intercepting $\gamma_{i,j}$ no earlier than $t$, which means $\taua$ intercepts $\gamma_{i',j'}$ no earlier than $t'$. This means $u'$ satisfies substitution condition \ref{cond:substitution3}. Furthermore, since $u''$ dominates $u'$, $t'' \leq t'$, so $\taua$ intercepts $\gamma_{i',j'}$ no earlier than $t''$, satisfying substitution condition \ref{cond:substitution3}.

    \item Let $k$ be the index of the first occurrence of $\gamma_{i,j}$ in $\Gamma^n$, let $\Gamma^n[:k]$ be the length-$k$ prefix of $\Gamma^n$, and let $\sigma_\text{dep}$ be the capacity depleted by $\Gamma^n[:k]$. Since $u$ is substitutable into $P_\text{sol}$, $\sigma \leq \sigma_\text{dep}$, so $\sigma' = \sigma + d_{i'} \leq \sigma_\text{dep} + d_{i'}$, where the RHS is the capacity depleted by $\Gamma^n[:k + 1]$. This means $u'$ satisfies substitution condition \ref{cond:substitution4}. Additionally, since $u''$ dominates $u'$, $\sigma'' \leq \sigma'$, so $\sigma'' \leq \sigma_\text{dep} + d_{i'}$, so $u''$ satisfies substitution condition \ref{cond:substitution4}.

    \item $\vec{b}'$ is identical to $\vec{b}$, except $\vec{b}'[i'] = 1$ while $\vec{b}[i'] = 0$. Since $u$ is substitutable into $P_\text{sol}$, for any target $i^\dagger \neq i'$ with $\vec{b}[i^\dagger] = 1$, $i^\dagger$ is visited by $\Gamma^n$ at or before $\gamma_{i,j}$ (and thus at or before $\gamma_{i',j'}$), or $i^\dagger$ is visited by some $\Gamma^j$ with $j < n$. For $i^\dagger = i'$, $\vec{b}[i^\dagger] = 1$, and $i^\dagger$ is visited by $\Gamma^n$ at $\gamma_{i',j'}$. Thus $u'$ and $u''$ both satisfy substitution condition \ref{cond:substitution5}.
\end{itemize}
Since $u''$ is in $\mathfrak{D}$, $u''$ must have been pushed onto the stack at some search iteration. Since the index in $\Gamma^n$ of its target-window $\gamma_{i',j'}$ is larger than the index of $\gamma_{i,j}$ (the target-window of $u$), $u'' \notin \Delta$, since we chose $u$ as the m-label in $\Delta$ with its target-window having largest index in $\Gamma^n$. This means one or more conditions among (i), (ii), and (iii) to be in $\Delta$ is not satisfied by $u''$. (ii) is satisfied, because we chose $u''$ so that it is not dominated by other m-labels at $(\gamma_{i',j'}, \vec{b}', n)$. (iii) is satisfied because $u''$ is substitutable into $P_\text{sol}$. Thus (i) is not satisfied, so $u''$ must still be on the stack. Thus we have an m-label on the stack that is substitutable into $P_\text{sol}$.

\underline{Case 2.2.2:} Suppose $u'$ was not dominated on Line \ref{algline:gen_feas_tours_dominate_before_push}. Then $u'$ was pushed onto the stack. Also, we proved in Case 2.2.1 that $u'$ satisfies substitution conditions \ref{cond:substitution1} to \ref{cond:substitution5}. This means we have an m-label on the stack that is substitutable into $P_\text{sol}$.
\end{proof}

\end{document}